\documentclass[11pt,letterpaper]{article}

\usepackage[margin=1in]{geometry}
\usepackage[T1]{fontenc}

\usepackage[english]{babel}
\usepackage{microtype}
\usepackage{booktabs}
\usepackage{tabularx}
\usepackage{xfrac}
\usepackage{amssymb}
\usepackage{mathtools}
\usepackage{amsthm}
\usepackage{float}
\usepackage{algorithm}
\usepackage{algpseudocode}

\usepackage[round,authoryear]{natbib}
\usepackage{hyperref}

\hypersetup{
  pdftitle={Convergence of Practical Muon with Finite Newton-Schulz Iterations and Nesterov Momentum},
  pdfkeywords={Muon, Newton--Schulz iteration, Nesterov momentum, nonconvex optimization}
}
\usepackage[capitalize,noabbrev]{cleveref}

\theoremstyle{plain}
\newtheorem{theorem}{Theorem}[section]
\newtheorem{proposition}[theorem]{Proposition}
\newtheorem{lemma}[theorem]{Lemma}
\newtheorem{corollary}[theorem]{Corollary}
\theoremstyle{definition}

\newtheorem{assumption}[theorem]{Assumption}
\theoremstyle{remark}
\newtheorem{remark}[theorem]{Remark}

\begin{document}

\title{Convergence of Practical Muon with Finite Newton--Schulz Iterations and Nesterov Momentum}
\author{Hanyang Peng, Hui Wang, Yue Yu\\Pengcheng Laboratory, Shenzhen, China\\\texttt{\{penghy, wangh06, yuy\}@pcl.ac.cn}}
\date{}
\date{}
\maketitle

\begin{abstract}
Practical Muon maintains momentum and performs a small, fixed number of Newton--Schulz iterations separately for each parameter matrix, often with a Nesterov correction. We analyze these layer-wise finite-step updates jointly on a coupled nonconvex objective, rather than replacing them by exact polar factors or one global orthogonalization. Under gradient-dependent $(\mathcal L_0,\mathcal L_1,q)$-smoothness and conditionally unbiased stochastic gradients with bounded layer-wise variance, we establish an $\mathcal O(T^{-1/4})$ bound on the expected average Frobenius gradient norm. The analysis retains the Nesterov recursion and requires neither bounded stochastic gradients, symmetric noise, nor a uniform positive lower bound on the nonzero output singular values. Its constants contain no explicit matrix-dimension or rank factors when the number of blocks and problem constants are fixed. The proof follows a descent inequality and a decomposition of the momentum tracking error into initialization, noise, and drift. For the original five-step quintic, we verify the required scalar-map bounds analytically; the result also allows step-dependent coefficients satisfying the same bounds. A complementary nuclear-norm result quantifies rank dependence under a stronger spectral condition. The vanishing rate uses coupled learning-rate and momentum schedules, including the standard single-coefficient Nesterov rule.
\end{abstract}

\section{Introduction}
\label{sec:introduction}
Muon orthogonalizes a momentum-based update separately for each matrix-valued
parameter. The reference implementation maintains a momentum buffer for each
parameter matrix, applies Frobenius normalization and a small number of quintic
Newton--Schulz iterations to that matrix, and includes an optional Nesterov
correction \citep{Jordan2024PracticalMuon}. These three features---layer-wise
updates, finite orthogonalization, and Nesterov momentum---all matter for
convergence analysis. A single global normalization is not the layer-wise
operation, a finite polynomial need not return an exact polar factor, and the
Nesterov correction changes the update's noise and gradient-tracking error.

Our analysis preserves this layer-wise structure. Each matrix block has its
own momentum, normalization, and orthogonalization, and the blocks may have
different dimensions and noise levels. Their updates nevertheless act on one
coupled objective: all stochastic gradient blocks are evaluated at the same
iterate, and changing one block may change the gradients in other blocks.
We therefore use joint smoothness to control simultaneous updates, while
retaining layer-specific variance bounds and tracking errors. The analysis
does not assume a separable loss or independent noise across layers.
Here a ``layer'' indexes a parameter matrix; several such matrices may belong
to the same architectural layer.

Existing analyses address different aspects of Muon.
\citet{Pethick2025NormConstrainedLMO} place non-Nesterov Muon in an exact-LMO
framework and develop layer-wise norm choices.
\citet{Riabinin2025Gluon} explicitly analyze layer-wise LMOs under block
generalized smoothness. \citet{Shen2025MuonConvergence} study exact-polar
Muon and its matrix geometry, whereas \citet{KimOh2026MuonNewtonSchulz}
analyze finite Newton--Schulz iterations through polar-approximation bounds.
\citet{Choudhury2026NesterovMuon} already combine Nesterov momentum with
inexact polar decomposition and allow gradient-dependent heavy-tailed noise.
\citet{LiTsuchiya2026FiniteMuon} exploit finite-iteration smoothing through
online-to-nonconvex conversion. Table~\ref{tab:practical_muon_comparison}
compares the algorithms, assumptions, and guarantees in these results.
Our contribution is not any one of these ingredients in isolation, but a
direct analysis of their combination: layer-wise Frobenius-normalized finite
polynomial updates with Nesterov momentum, under joint gradient-dependent
smoothness and bounded layer-wise conditional variance, yielding an average
Frobenius stationarity bound without explicit matrix-size or rank factors.
The assumptions are not uniformly weaker than those of all prior work;
in particular, our bounded-variance model does not cover the heavier-tailed
regimes of \citet{Choudhury2026NesterovMuon}.

\begin{table}[H]
\centering
\caption{Comparison of representative convergence analyses relevant to practical Muon.
NS denotes Newton--Schulz iteration; LMO denotes a linear minimization oracle.
Entries describe the cited theoretical results, not every implementation studied in each paper.}
\label{tab:practical_muon_comparison}
\begingroup
\footnotesize
\setlength{\tabcolsep}{2pt}
\renewcommand{\arraystretch}{1.16}
\begin{tabularx}{\textwidth}{@{}
>{\raggedright\arraybackslash}p{0.137\textwidth}
>{\raggedright\arraybackslash}p{0.110\textwidth}
>{\raggedright\arraybackslash}p{0.174\textwidth}
>{\raggedright\arraybackslash}p{0.085\textwidth}
>{\raggedright\arraybackslash}p{0.224\textwidth}
>{\raggedright\arraybackslash}X@{}}
\toprule
\textbf{Work} & \textbf{Block structure} & {\scriptsize\textbf{Orthogonalization in theory}}
& {\scriptsize\textbf{Nesterov}} & \textbf{Smoothness and noise}
& \textbf{Stationarity and dimension factors} \\
\midrule
\citet{Pethick2025NormConstrainedLMO}\newline (Scion/uSCG)
& General-norm theorem; layer-wise design
& Exact LMO; polar factor for Muon
& No\textsuperscript{a}
& Norm smoothness; bounded Euclidean variance
& Dual-norm gradient; geometry-dependent constants \\
\addlinespace[4pt]
\citet{Riabinin2025Gluon}\newline (Gluon)
& Explicitly layer-wise
& Exact block LMOs
& No
& Block generalized smoothness; block dual-norm variance
& Weighted block dual norms; norm-specific constants \\
\addlinespace[4pt]
\citet{Shen2025MuonConvergence}
& Single matrix
& Exact SVD-polar factor
& No
& Frobenius or spectral smoothness; bounded Frobenius variance
& Frobenius/nuclear gradient; rank factors in stochastic bounds \\
\addlinespace[4pt]
\citet{KimOh2026MuonNewtonSchulz}
& Single matrix
& Finite Taylor NS; polar-error analysis
& No
& Spectral smoothness; bounded Frobenius variance
& Average nuclear gradient; explicit rank factors \\
\addlinespace[4pt]
\citet{Choudhury2026NesterovMuon}
& Single matrix
& Inexact polar, including NS; relative alignment bound
& Yes
& Frobenius $L$-smoothness; gradient-dependent $\alpha$-moment noise, $1<\alpha\le2$
& Best-iterate Frobenius gradient; $d_0$-dependent constants\textsuperscript{b} \\
\addlinespace[4pt]
\citet{LiTsuchiya2026FiniteMuon}
& Single-matrix online learner
& Finite Taylor NS; fixed normalization
& No
& Nonsmooth objectives allowed; moment bounds and almost-sure operator-norm gradient bound
& Online-to-nonconvex stationarity; rank-dependent bounds \\
\midrule
\textbf{This work}\newline Theorem~\ref{thm:muon_dimension_independent}
& \textbf{Explicitly layer-wise}
& \textbf{Fixed finite NS}; includes the practical five-step quintic
& \textbf{Yes}
& Joint $(\mathcal L_0,\mathcal L_1,q)$-smoothness; layer-wise conditional Frobenius variance
& \textbf{Average Frobenius gradient}, $\mathcal O(T^{-1/4})$; \textbf{no size/rank factors}\textsuperscript{c} \\
\bottomrule
\end{tabularx}
\vspace{4pt}
\begin{minipage}{\textwidth}
\scriptsize
\textsuperscript{a}The Muon specialization in the uSCG theory uses non-Nesterov momentum; Scion also develops layer-wise norm choices.
\textsuperscript{b}$d_0=\min\{m,n\}$.
\textsuperscript{c}With fixed block count, NS depth and coefficients, and problem constants; the vanishing rate uses coupled momentum and learning-rate schedules.
``Single matrix'' describes the theorem's formulation, not an impossibility of extending it.
The norms, noise models, and output criteria differ, so the rows are not ordered by a common strength of assumptions.
For our implementation scope, see Remark~\ref{rem:df_normalization}.
\end{minipage}
\endgroup
\end{table}

The main obstacle is not only the approximation error relative to a polar
factor. For a fixed number $J$ of iterations, let $\phi_J$ be the composite
scalar polynomial acting on normalized singular values. Since
$\phi_J(s)\to0$ as $s\downarrow0$, a uniform positive lower bound on all
nonzero output singular values is unavailable over arbitrary spectra.
However, the ratio $\phi_J(s)/s$ can remain bounded above and away from zero.
Together with Frobenius normalization of each block, this yields descent and
update-energy bounds that do not introduce a rank factor. We use these bounds
without replacing the finite map by an exact polar direction.

The resulting theorem controls the expected average Frobenius gradient norm
along the original stochastic iterates. Its proof retains the same three
parts of each block's Nesterov tracking error: the initial momentum bias,
the accumulated stochastic noise, and the drift of the true gradient under
joint updates. Layer-specific noise levels remain visible through
$S_{\mathsf F}=\sum_{l=1}^L\sigma_{(l)}$, rather than through a matrix-rank
conversion. The $\mathcal O(T^{-1/4})$ rate follows by coupling the learning
rate and momentum parameters. Taking the same coefficient in the momentum
and Nesterov correction is included, not removed for the analysis. For the
original quintic, the scalar-map requirements hold for every fixed $J\ge1$,
in particular for $J=5$; they can also be checked for schedules using
different coefficients at different steps.

Section~\ref{sec:main_results} develops these results within one proof
framework. It first gives a rank-sensitive nuclear-norm bound under a
positive output singular-value floor, then removes that condition by using
the finite map and measuring stationarity in Frobenius norm. The latter
result uses neither a bounded-gradient nor a noise-symmetry assumption.
Dimension independence here concerns the optimization bound with fixed
block count and problem constants; it is not a claim about how those
constants change across neural-network architectures.

\section{Preliminary}
\label{sec:preliminary}
We minimize $F(X)=\mathbb E_\xi[f(X;\xi)]$ over matrix blocks
$X=(X^{(1)},\ldots,X^{(L)})$. Algorithm~\ref{alg:muon_layerwise} uses
$M_t^{(l)}$ for the first-order momentum and $N_t^{(l)}$ for the direction
passed to the orthogonalization routine. The setting $\beta_2=1$ gives
ordinary momentum, whereas $\beta_2=\beta_1=\beta$ gives
\begin{equation}
\label{eq:practical_nesterov}
N_t^{(l)}=\beta M_t^{(l)}+(1-\beta)G_t^{(l)},
\end{equation}
the single-coefficient Nesterov correction used in the reference
implementation.\footnote{The reference code uses an exponential moving
average for the momentum and a second interpolation with the same
coefficient: \url{https://github.com/KellerJordan/Muon/blob/master/muon.py}.}
The routine may be inexact; its finite Newton--Schulz specialization is
specified in Section~\ref{sec:finite_spectral_map}.

Layer-wise refers to the separate matrix operations in
Algorithm~\ref{alg:muon_layerwise}: the normalization in each call to
orthogonalization uses that block's $N_t^{(l)}$, not the norm of the full
momentum tuple. All $G_t^{(l)}$ are evaluated at $X_t$ before the joint iterate
$X_{t+1}$ is formed. This is not cyclic block-coordinate descent, and no
cross-block derivatives are set to zero. The common sample $\xi_t$ may also
correlate the noise across blocks.

Here ``practical'' refers to retaining the layer-wise matrix updates,
finite Newton--Schulz map, and Nesterov correction, rather than analyzing an
exact-polar or globally orthogonalized surrogate. The
mathematical iteration is studied in exact arithmetic, without decoupled
weight decay or a separate optimizer for other parameter groups.
Remark~\ref{rem:df_normalization} treats normalization stabilizers and fixed
output scaling separately.

\renewcommand{\algorithmiccomment}[1]{\hfill $\triangleright$ #1}

\begin{algorithm*}[t]
\caption{Practical layer-wise Muon with Nesterov momentum and inexact orthogonalization}
\label{alg:muon_layerwise}
\begin{algorithmic}[1]
\Require learning rate $\gamma > 0$, momentum $\beta_1\in[0,1)$, $\beta_2\in[0,1]$, total iterations $T$, number of layers $L$.
\State \textbf{Initialize:} $M_{0}^{(l)} \leftarrow 0$, $X_{1}^{(l)} \in \mathcal{S}_l$ for each layer $l \in \{1, \dots, L\}$.
\For{$t = 1$ to $T$}
    \For{each layer $l = 1$ to $L$}
        \State $G_{t}^{(l)} \leftarrow \nabla_{(l)} f(X_{t}; \xi_t)$ \Comment{Compute layer-wise batch gradients}
        \State $M_{t}^{(l)} \leftarrow \beta_1 M_{t-1}^{(l)} + (1-\beta_1)G_{t}^{(l)}$ \Comment{Update first-order momentum}
        \State $N_{t}^{(l)} \leftarrow \beta_2 M_{t}^{(l)} + (1-\beta_2) G_{t}^{(l)}$ \Comment{Apply Nesterov correction}
        \State $O_{t}^{(l)} \leftarrow \mathrm{Inexact\text{-}Orthogonalization} (N_{t}^{(l)})$ \Comment{Generate approximate polar proxy}
        \State $X_{t+1}^{(l)} \leftarrow X_{t}^{(l)} - \gamma O_{t}^{(l)}$ \Comment{Update parameters} 
    \EndFor 
\EndFor
\end{algorithmic} 
\end{algorithm*}

\begin{proposition}[Spectral Structure of Inexact Orthogonalization]
\label{prop:spectral_structure}
Let the compact SVD of the extrapolated momentum be $N_t = U \Sigma V^\top \in \mathbb{R}^{m \times n}$ with $\Sigma = \mathrm{diag}(\sigma_1, \dots, \sigma_r)$, where $r = \mathrm{rank}(N_t)$. A singular-vector-preserving orthogonalization routine gives $O_t = U \boldsymbol{\varsigma} V^\top$, where $\boldsymbol{\varsigma} = \mathrm{diag}(\varsigma_1, \dots, \varsigma_r)$. For a finite polynomial routine with input normalization $\nu_t>0$, the modified values are $\varsigma_i=\phi_J(\sigma_i/\nu_t)$, where $J$ denotes the number of polynomial steps and is distinct from the smoothness exponent $q$. Suppose these values are nonnegative. Then:
\begin{enumerate}
    \item \textbf{Directional Alignment:} $\langle N_t, O_t \rangle \ge \varsigma_{\min} \| N_t \|_*$, where $\varsigma_{\min} = \min_i \varsigma_i$.
    \item \textbf{Frobenius Norm Bound:} $\| O_t \|_\mathsf{F} \le \varsigma_{\max}\sqrt{r}$, where $\varsigma_{\max} = \max_i \varsigma_i$.
\end{enumerate}
If $N_t=0$, take $O_t=0$ and interpret both inequalities directly, without a minimum over an empty spectrum. Uniform positive lower bounds on the nonzero $\varsigma_i$ are an additional condition in Theorem~\ref{thm:muon_convergence}; they do not follow merely from using finitely many Newton--Schulz steps.
\end{proposition}

\begin{assumption}[Objective Lower Boundedness]
\label{assump:lower_boundedness}
The objective function $F$ is lower bounded on $\mathcal{S}$; that is, there exists a constant $F^*>-\infty$ such that
\begin{equation}
F(X)\ge F^*, \qquad \forall X\in\mathcal{S}.
\end{equation}
\end{assumption}

\begin{assumption}[Global $(\mathcal{L}_0, \mathcal{L}_1, q)$-Smoothness]
\label{assump:layer_wise_smoothness}
Let $\mathcal{S} = \mathbb{R}^{n_1 \times m_1} \times \dots \times \mathbb{R}^{n_L \times m_L}$, and equip the complete parameter tuple with
$\|X\|_{\mathsf F}^2=\sum_{l=1}^L\|X^{(l)}\|_{\mathsf F}^2$.
The differentiable objective $F:\mathcal S\to\mathbb R$ satisfies, for some $\mathcal L_0>0$, $\mathcal L_1\ge0$, and $q\in[0,1]$, and all $X,Y\in\mathcal S$,
\begin{equation}
\begin{aligned}
&\|\nabla F(X)-\nabla F(Y)\|_{\mathsf F}\\
&\quad\le \bigl(\mathcal L_0+\mathcal L_1\|\nabla F(X)\|_{\mathsf F}^{q}\bigr)
\|X-Y\|_{\mathsf F}.
\end{aligned}
\end{equation}
For $q=0$, the power is interpreted as $1$, including at a zero gradient. The condition applies when several blocks change simultaneously, as in Algorithm~\ref{alg:muon_layerwise}. A bound only for $X,Y$ differing in one block would not justify the joint descent and momentum-drift estimates used below.
\end{assumption}

\begin{assumption}[Layer-wise Variance]
\label{assump:layer_wise_variance}
The stochastic gradient oracle is assumed to have a layer-wise variance bound, where the parameter space $\mathcal{S} = \mathbb{R}^{n_1 \times m_1} \times \dots \times \mathbb{R}^{n_L \times m_L}$ is the Cartesian product of the weight spaces of all $L$ layers. Specifically, for each layer $l \in \{1, \dots, L\}$, there exists a constant $\sigma_{(l)} \ge 0$ such that, for any $X \in \mathcal{S}$, the noisy gradient $\nabla_{(l)} f(X; \xi)$ with respect to the $l$-th layer satisfies:
\begin{equation}
\begin{split}
&\mathbb{E}_{\xi} \left[ \nabla_{(l)} f(X; \xi) \right] = \nabla_{(l)} F(X), \\
&\mathbb{E}_{\xi} \left[ \left\| \nabla_{(l)} f(X; \xi) - \nabla_{(l)} F(X) \right\|_{\mathsf{F}}^2 \right] \le \sigma_{(l)}^2,
\end{split}
\end{equation}
where $\xi$ denotes the random sampling index and $\|\cdot\|_{\mathsf{F}}$ denotes the Frobenius norm.
At iteration $t$, $\xi_t$ is sampled independently of the past $\mathcal F_{t-1}$, which includes $X_t$. Thus the same statements hold conditionally on $\mathcal F_{t-1}$. Independence of the noise across layers is not required.
\end{assumption}

\paragraph{Scope of the assumptions.}
Assumption~\ref{assump:layer_wise_smoothness} includes ordinary global
$\mathcal L_0$-smoothness when $\mathcal L_1=0$ and otherwise allows the
smoothness bound to depend on the gradient norm. We use the joint condition
because all blocks are updated together. Assumption~\ref{assump:layer_wise_variance}
controls conditional variance, not the magnitude of every stochastic
gradient; no almost-sure gradient bound, noise symmetry, or sub-Gaussian
tail assumption is imposed. The finite-step result below further dispenses
with a positive lower bound on nonzero output singular values. These are
the specific relaxations used in the analysis; no additional architecture
assumption is needed.

\section{Main Results}
\label{sec:main_results}
Throughout, $\nabla F_{(l)}(X)$ and $\nabla_{(l)}F(X)$ denote the same gradient
block. We develop one descent--tracking argument for
Algorithm~\ref{alg:muon_layerwise}, first under output singular-value bounds
and then with the geometry of the finite Newton--Schulz map.
The block sums retain layer-specific noise levels, while gradient drift is
controlled for the joint move $X_{t+1}-X_t$ on the coupled objective.

\subsection{A common descent and tracking argument}
\label{sec:rank_sensitive_convergence}
\label{sec:finite_spectral_map}
\label{sec:dimension_independence}
The proof has four steps: derive and telescope a descent inequality;
expand the Nesterov tracking error; bound its initialization, noise, and
gradient-drift terms; and substitute these bounds to absorb the remaining
gradient term. The orthogonalization routine enters through alignment and
update-energy bounds. We establish these bounds in two forms below.

\paragraph{Control through output singular values.}
The first form uses Proposition~\ref{prop:spectral_structure} and gives a
nuclear-norm stationarity bound with explicit rank dependence. The positive
output singular-value floor is an additional spectral condition. The
finite-map estimates later in this subsection lead to the Frobenius result
in Section~\ref{sec:dimension_independent_convergence} without that condition.

\begin{theorem}[Convergence of Layer-wise Muon]
\label{thm:muon_convergence}
Suppose that Assumptions~\ref{assump:lower_boundedness},~\ref{assump:layer_wise_smoothness}, and~\ref{assump:layer_wise_variance} hold. Let $\{X_t\}_{t=1}^{T+1}$ be generated by Algorithm~\ref{alg:muon_layerwise} with $\beta_1\in[0,1), \beta_2 \in [0,1]$ and $\gamma>0$. For each layer $l$, let $r^{(l)}$ bound both $\mathrm{rank}(N_t^{(l)})$ and $\mathrm{rank}(N_t^{(l)}-\nabla_{(l)}F(X_t))$ along the iterates (the choice $r^{(l)}=\min\{n_l,m_l\}$ always suffices), and let the inexact orthogonalization satisfy Proposition~\ref{prop:spectral_structure} with fixed constants $\varsigma_{\min}>0$ and $\varsigma_{\max}>0$ uniformly over all layers and iterates. Define $\hat{\mathcal{L}}=\mathcal{L}_0+(1-q)\mathcal{L}_1$, $R_0=\sum_{l=1}^L r^{(l)}$, $S_0=\sum_{l=1}^L\sqrt{r^{(l)}}\sigma_{(l)}$, $g_l=\sqrt{r^{(l)}}\left\Vert\nabla F_{(l)}(X_1)\right\Vert_{\mathsf{F}}$, and $G_0=\sum_{l=1}^L g_l$.
Set
\[
\begin{aligned}
\Delta_{\gamma,\beta_1}:=\varsigma_{\min}
-\frac{\gamma\varsigma_{\max}^2 q\mathcal{L}_1R_0}{2}
-\frac{2\gamma\varsigma_{\max}^2q\mathcal{L}_1R_0}{1-\beta_1}.
\end{aligned}
\]
Then
\begin{equation}
\label{eq:muon_convergence}
\begin{aligned}
&\Delta_{\gamma,\beta_1}
\frac{1}{T}\sum_{t=1}^{T}\sum_{l=1}^L
{\mathbb E}\left[\left\Vert \nabla F_{(l)}(X_t) \right\Vert_*\right] \\
&\le
\frac{F(X_{1})-F^*}{\gamma T}
+ \frac{2\beta_2\varsigma_{\max}G_0}{T(1-\beta_1)} \\
&\;
+ 2\left(
\sqrt{1-\beta_1}
+\sqrt{2\beta_1(1-\beta_1\beta_2)(1-\beta_2)}
\right)\varsigma_{\max}S_0 \\
&\;
+ \frac{2\gamma\beta_2\varsigma_{\max}^2\hat{\mathcal{L}}}{1-\beta_1}
R_0
+ \frac{\gamma\varsigma_{\max}^2\hat{\mathcal{L}}R_0}{2}.
\end{aligned}
\end{equation}
\end{theorem}

\begin{corollary}[Rate of Layer-wise Muon with scheduled Nesterov correction]
\label{cor:muon_convergence}
Suppose that the conditions of Theorem~\ref{thm:muon_convergence} hold, with $F(X_1)-F^*>0$ and $S_0>0$ for the optimized parameter formulas below. Denote
$R_0=\sum_{l=1}^L r^{(l)}$,
$S_0=\sum_{l=1}^L\sqrt{r^{(l)}}\sigma_{(l)}$, and
$G_0=\sum_{l=1}^L\sqrt{r^{(l)}}\left\Vert\nabla F_{(l)}(X_1)\right\Vert_{\mathsf{F}}$.
Let
$\beta_2 = 1-(1-\beta_1)^a$ for $a>\frac{1}{2}$. Since $\beta_2\le1$, choose
$C_1,C_2,\gamma,\beta_1$ by minimizing the leading upper envelope
\[
\begin{aligned}
\Phi(\gamma,\beta_1):=&\frac{F(X_1)-F^*}{\gamma T}
+2\sqrt{1-\beta_1}\varsigma_{\max}S_0\\
&+\frac{2\gamma\varsigma_{\max}^2\hat{\mathcal{L}}R_0}{1-\beta_1},
\end{aligned}
\]
whose Young equality condition is
\[
\begin{aligned}
\frac{4(F(X_1)-F^*)}{\gamma T}
=4\sqrt{1-\beta_1}\varsigma_{\max}S_0
=\frac{8\gamma\varsigma_{\max}^2\hat{\mathcal{L}}R_0}{1-\beta_1}.
\end{aligned}
\]
This gives
\[
\begin{aligned}
C_1 &= \frac{(F(X_1)-F^*)^{\sfrac{3}{4}}}
{2^{\sfrac{1}{4}}\varsigma_{\max}S_0^{\sfrac{1}{2}}\hat{\mathcal{L}}^{\sfrac{1}{4}}R_0^{\sfrac{1}{4}}},\\
C_2 &=
\frac{2^{\sfrac{1}{2}}\hat{\mathcal{L}}^{\sfrac{1}{2}}(F(X_1)-F^*)^{\sfrac{1}{2}}R_0^{\sfrac{1}{2}}}{S_0},\\
\gamma &=
\frac{C_1}{T^{\sfrac{3}{4}}},
\qquad
\beta_1 = 1-\frac{C_2}{T^{\sfrac{1}{2}}}.
\end{aligned}
\]
Assume that $T$ is large
enough such that $\beta_1\in[0,1)$ and
\begin{equation}
\label{eq:original_horizon}
T \ge
\left(
\left(
\frac{C_1\varsigma_{\max}^2q\mathcal{L}_1R_0}
{\varsigma_{\min}}
\right)^{\sfrac{1}{3}}
+
\frac{4C_1\varsigma_{\max}^2q\mathcal{L}_1R_0}
{C_2\varsigma_{\min}}
\right)^4 .
\end{equation}
Then the iterates generated by Algorithm~\ref{alg:muon_layerwise} satisfy
\begin{equation}
\begin{aligned}
&\frac{1}{T}\sum_{t=1}^{T}\sum_{l=1}^L
{\mathbb E}\left[\left\Vert\nabla F_{(l)}(X_t)\right\Vert_*\right]
\le \frac{2\cdot512^{\sfrac{1}{4}}\varsigma_{\max}
\hat{\mathcal{L}}^{\sfrac{1}{4}}S_0^{\sfrac{1}{2}}R_0^{\sfrac{1}{4}}}{\varsigma_{\min}T^{\sfrac{1}{4}}}
\\
&\quad \cdot (F(X_1)-F^*)^{\sfrac{1}{4}}
+ \frac{4\varsigma_{\max}G_0}{C_2\varsigma_{\min}T^{\sfrac{1}{2}}}
\\
&\quad + \frac{4\sqrt{2}\varsigma_{\max}S_0}
{\varsigma_{\min}}
\left(\frac{ C_2^{\sfrac{(1+a)}{2}}}{T^{\sfrac{(1+a)}{4}}}
+\frac{ C_2^{a}}{T^{\sfrac{a}{2}}}\right)
\\
&\quad + \frac{C_1\varsigma_{\max}^2\hat{\mathcal{L}}R_0}
{\varsigma_{\min}T^{\sfrac{3}{4}}}.
\end{aligned}
\end{equation}
Consequently,
\begin{equation}
\frac{1}{T}\sum_{t=1}^{T}\sum_{l=1}^L
{\mathbb E}\left[\left\Vert\nabla F_{(l)}(X_t)\right\Vert_*\right]
=\mathcal{O}\left(T^{-\sfrac{1}{4}}\right).
\end{equation}
\end{corollary}

\paragraph{Control through the finite Newton--Schulz map.}
To use the same proof for a fixed number of polynomial steps, we derive
alignment and energy bounds from the actual scalar map. Its output can
approach zero on small input singular values, so the preceding theorem's
positive output floor is not generally available. The ratio of output to
input singular values provides the required control instead.
For $N\ne0$, initialize $Z_0=N/\|N\|_{\mathsf F}$ and take a fixed number
$J$ of polynomial steps
\begin{equation}
\label{eq:finite_ns_matrix}
\begin{split}
Z_j={}&a_jZ_{j-1}+b_j(Z_{j-1}Z_{j-1}^{\top})Z_{j-1}\\
&+c_j(Z_{j-1}Z_{j-1}^{\top})^2 Z_{j-1},
\qquad j=1,\ldots,J,
\end{split}
\end{equation}
with $O=Z_J$; set $O=0$ when $N=0$. The coefficients may vary with $j$,
but $J$ and their values are fixed independently of matrix size and $T$.
All statements concern exact arithmetic. Write
\begin{equation}
\label{eq:finite_ns_scalar}
p_j(s)=a_js+b_js^3+c_js^5,\qquad
\phi_J=p_J\circ\cdots\circ p_1.
\end{equation}
We require the following property of this scalar map:
\begin{equation}
\label{eq:finite_ns_constants}
\begin{aligned}
0<\ell_J&:=\inf_{0<s\le1}\frac{\phi_J(s)}s,
&u_J&:=\sup_{0<s\le1}\frac{\phi_J(s)}s<\infty,\\
c_J&:=\max_{0\le s\le1}\phi_J(s).
\end{aligned}
\end{equation}
The ratio has the continuous extension
$\phi_J'(0)=\prod_{j=1}^Ja_j$ at zero. Thus these are scalar constants, not
bounds on a matrix's smallest nonzero singular value. Positive lower bounds
and finite upper bounds can also be used in place of the extrema in
\eqref{eq:finite_ns_constants}.

\begin{proposition}[Finite-map alignment and energy]
\label{prop:finite_map_geometry}
Let $N=U\Sigma V^\top\ne0$, $R=\|N\|_{\mathsf F}$, and
$s_i=\sigma_i(N)/R$. Then $\sum_i s_i^2=1$ and the direction in
\eqref{eq:finite_ns_matrix} is
$O=U\operatorname{diag}(\phi_J(s_i))V^\top$. Define
\begin{equation}
\label{eq:spectral_AB}
A(N)=\sum_i s_i\phi_J(s_i),\qquad
B(N)=\sum_i\phi_J(s_i)^2.
\end{equation}
The exact identities and bounds are
\begin{equation}
\label{eq:finite_map_energy}
\begin{aligned}
\langle N,O\rangle&=R A(N),&\quad \|O\|_{\mathsf F}^2&=B(N),\\
\ell_J\le A(N)&\le u_J,& B(N)&\le u_J A(N)\le u_J^2,\\
\|O\|_{\mathrm{op}}&\le c_J.&
\end{aligned}
\end{equation}
Moreover, with
\begin{equation}
\label{eq:angular_constant}
\bar\rho_J:=\frac{u_J+\ell_J}{2\sqrt{u_J\ell_J}},
\end{equation}
we have $\sqrt{B(N)}\le\bar\rho_J A(N)$. Consequently, for any matrix
$H$ and $E=N-H$,
\begin{equation}
\label{eq:full_shape_alignment}
\langle H,O\rangle\ge
A(N)\bigl(\|N\|_{\mathsf F}-\bar\rho_J\|E\|_{\mathsf F}\bigr).
\end{equation}
No rank factor appears in these statements.
\end{proposition}

\noindent\textbf{An explicit five-step instance.}\par
For the fixed quintic coefficients
$(a,b,c)=(3.4445,-4.775,2.0315)$, as used in the original Muon
implementation \citep{Jordan2024PracticalMuon},
one may use the fully analytic bounds
\begin{equation}
\label{eq:explicit_ns_constants}
\begin{aligned}
\ell_J&\ge\frac{17}{25},&
u_J&=3.4445^J,\\
c_J&\le\frac54,&J&\ge1.
\end{aligned}
\end{equation}
The lower bound uses invariant intervals of the whole scalar map; it is
stronger than multiplying a separate worst-case lower bound at every step.
For $J=5$, $u_5=3.4445^5\simeq484.8763$, and the displayed bounds imply
$\bar\rho_5<13.371$. These are certified bounds, not an assertion that a
numerical scan has found the sharp lower or angular constant. The details
are given in Appendix~\ref{app:explicit_ns_constants}.

The bounds $\langle N,O\rangle\ge\ell_J\|N\|_{\mathsf F}$ and
$\|O\|_{\mathsf F}\le u_J$ now replace the nuclear-norm alignment and
rank-dependent update bound in the same four-step argument. The Nesterov
error expansion is unchanged; joint gradient drift is controlled using
$\|X_{t+1}-X_t\|_{\mathsf F}\le\gamma u_J\sqrt L$.

\subsection{Dimension-independent convergence of practical Muon}
\label{sec:dimension_independent_convergence}
We apply the argument of Section~\ref{sec:rank_sensitive_convergence}
with its finite-map estimates. The resulting stationarity measure is the
sum of block Frobenius norms, which also bounds the Frobenius norm of the
complete gradient. No positive output singular-value floor or rank bound
for the momentum or tracking error is required.

\begin{theorem}[Dimension-independent Frobenius stationarity]
\label{thm:muon_dimension_independent}
Suppose Assumptions~\ref{assump:lower_boundedness},
\ref{assump:layer_wise_smoothness}, and
\ref{assump:layer_wise_variance} hold. Use
\eqref{eq:finite_ns_matrix} in Algorithm~\ref{alg:muon_layerwise}, with
\eqref{eq:finite_ns_constants}, $\beta_1\in[0,1)$,
$\beta_2\in[0,1]$, and $\gamma>0$. Define
\begin{equation}
\label{eq:df_notation}
\begin{aligned}
\delta&=1-\beta_1,\qquad D=F(X_1)-F^*,\\
S_{\mathsf F}&=\sum_{l=1}^L\sigma_{(l)},\\
G_{\mathsf F}&=\sum_{l=1}^L\|\nabla_{(l)}F(X_1)\|_{\mathsf F},\\
v_\beta&=\sqrt{\delta}
 +\sqrt{2\beta_1(1-\beta_1\beta_2)(1-\beta_2)},\\
\hat{\mathcal L}&=\mathcal L_0+(1-q)\mathcal L_1.
\end{aligned}
\end{equation}
Set
\begin{equation}
\label{eq:df_absorption}
\Delta^{\mathsf F}_{\gamma,\beta_1}:=
\ell_J-\frac{\gamma L u_J^2 q\mathcal L_1}{2}
-\frac{\gamma L u_J(\ell_J+u_J)q\mathcal L_1}{\delta}.
\end{equation}
Then
\begin{equation}
\label{eq:df_convergence}
\begin{aligned}
&\Delta^{\mathsf F}_{\gamma,\beta_1}\,
\frac1T\sum_{t=1}^T\sum_{l=1}^L
\mathbb E\|\nabla_{(l)}F(X_t)\|_{\mathsf F}\\
&\le\frac{D}{\gamma T}
+\frac{\beta_2(\ell_J+u_J)G_{\mathsf F}}{T\delta}
+(\ell_J+u_J)v_\beta S_{\mathsf F}\\
&\quad+\frac{\gamma\beta_2 L u_J(\ell_J+u_J)\hat{\mathcal L}}{\delta}
+\frac{\gamma L u_J^2\hat{\mathcal L}}{2}.
\end{aligned}
\end{equation}
In particular, a positive $\Delta^{\mathsf F}_{\gamma,\beta_1}$ gives a
bound on the original, unstopped sequence of iterates.
\end{theorem}

\begin{corollary}[Dimension-independent $T^{-1/4}$ rate]
\label{cor:muon_dimension_independent}
Under Theorem~\ref{thm:muon_dimension_independent}, fix
$\gamma_0,\delta_0>0$ and $a>1/2$, and set
\begin{equation}
\label{eq:df_schedule}
\gamma=\gamma_0T^{-3/4},\quad
1-\beta_1=\delta_0T^{-1/2},\quad
\beta_2=1-(1-\beta_1)^a.
\end{equation}
For all sufficiently large $T$, $\beta_1,\beta_2\in[0,1]$ and
$\Delta^{\mathsf F}_{\gamma,\beta_1}\ge\ell_J/2$. Writing $b_J=\ell_J+u_J$,
\begin{equation}
\label{eq:df_rate_explicit}
\begin{aligned}
&\frac1T\sum_{t=1}^T\sum_{l=1}^L
\mathbb E\|\nabla_{(l)}F(X_t)\|_{\mathsf F}\\
&\le\frac{2}{\ell_J}\Bigg[
\left(\frac D{\gamma_0}+b_JS_{\mathsf F}\sqrt{\delta_0}
+\frac{\gamma_0L u_J b_J\hat{\mathcal L}}{\delta_0}\right)T^{-1/4}\\
&\qquad+\frac{b_JG_{\mathsf F}}{\delta_0T^{1/2}}\\
&\qquad+\sqrt2 b_JS_{\mathsf F}\left(
\frac{\delta_0^{(1+a)/2}}{T^{(1+a)/4}}
+\frac{\delta_0^a}{T^{a/2}}\right)\\
&\qquad+\frac{\gamma_0L u_J^2\hat{\mathcal L}}{2T^{3/4}}
\Bigg].
\end{aligned}
\end{equation}
The same upper bound holds for
$T^{-1}\sum_t\mathbb E\|\nabla F(X_t)\|_{\mathsf F}$.
Its constant contains no $n_l$, $m_l$, $r^{(l)}$, or
$d=\sum_l n_lm_l$.
\end{corollary}

\paragraph{The standard Nesterov choice is included.}
Setting $a=1$ in \eqref{eq:df_schedule} gives $\beta_2=\beta_1$ and hence
\eqref{eq:practical_nesterov}. In this case,
\begin{equation}
\label{eq:practical_nesterov_noise}
v_\beta=\sqrt\delta+\delta\sqrt{2\beta_1(1+\beta_1)}
\le\sqrt\delta+2\delta.
\end{equation}
Thus the extra Nesterov noise term is $\mathcal O(T^{-1/2})$, smaller than
the leading $\mathcal O(T^{-1/4})$ term. This specialization uses the usual
Nesterov algebra, with a horizon-dependent momentum coefficient; it does
not assert the same vanishing bound for a fixed numerical momentum
coefficient and a fixed batch size.

A sufficient explicit horizon in this corollary is
\begin{equation}
\label{eq:df_horizon}
\begin{aligned}
T&\ge\max\{1,\delta_0^2,(A_J+B_J^{1/3})^4\},\\
A_J&=\frac{2\gamma_0L u_J(\ell_J+u_J)q\mathcal L_1}
{\delta_0\ell_J},\qquad
B_J=\frac{\gamma_0L u_J^2q\mathcal L_1}{\ell_J}.
\end{aligned}
\end{equation}
For $q\mathcal L_1=0$, the absorption condition is automatic. The
initial-gradient constant can also be eliminated: joint generalized
smoothness and lower boundedness imply
\begin{equation}
\label{eq:df_initial_gradient}
G_{\mathsf F}\le\sqrt L\left[
Dq\mathcal L_1+
\sqrt{D^2q^2\mathcal L_1^2+2D\hat{\mathcal L}}\right].
\end{equation}
Thus, for fixed $L,J$, coefficient schedule, $\gamma_0,\delta_0,a$,
$\mathcal L_0$, $\mathcal L_1$, $q$, $D$, and $S_{\mathsf F}$, the complete displayed
bound is independent of parameter dimension. This is not a proof that
these problem quantities stay fixed when a model family changes.

\section{Discussion}
\label{sec:discussion}
We discuss the coupled momentum and learning-rate schedules, the meaning of
dimension independence, and the scope of the orthogonalization and
normalization choices covered by the analysis.

\subsection{Momentum schedules}
\label{sec:momentum_schedules}

\begin{remark}[Coupled Muon schedules versus fixed momentum]
\label{rem:coupled_muon_scaling}
There is a growing body of evidence that stable and transferable learning
rates are not architecture-agnostic.  The $\mu$P framework gives width-aware
hyperparameter transfer rules for large neural networks
\citep{MuP_2022}, and recent Muon pretraining studies combine Muon with
$\mu$P-style transfer and carefully chosen matrix update scales
\citep{Moun_LLM_2025,Moun_LLM_2025_1}.  For SGD/GD-style training, prior
analyses have also identified nontrivial depth dependence of maximal or
effective learning rates, including maximal initial learning rates in deep
ReLU networks \citep{Iyer2023MaximalLR}, depth-dependent $\mu$P learning rates
\citep{Jelassi2023DepthMuP}, and architecture-aware learning-rate rules that
depend on depth, width, kernel size, and graph topology
\citep{Chen2024ArchAwareHyperparams}.  In Transformer training with Adam,
\citet{LN_mean_field_2020} also showed that depth-dependent gradient behavior
is closely related to learning-rate warmup and normalization design.

The rank-sensitive bound makes explicit how matrix structure enters one
sufficient Muon stepsize.  Corollary~\ref{cor:muon_convergence} should be interpreted as a
theoretical coupled hyperparameter schedule for Muon.  The learning-rate rule
\[
\gamma\asymp
\frac{1}{S_0^{\sfrac{1}{2}}R_0^{\sfrac{1}{4}}T^{\sfrac{3}{4}}}
\]
is obtained together with the momentum schedule
\[
1-\beta_1\asymp \frac{R_0^{\sfrac{1}{2}}}{S_0\sqrt{T}},
\qquad
1-\beta_2=(1-\beta_1)^a,\quad a>\frac{1}{2}.
\]
Thus, the depth-rank dependence of $\gamma$ is not a stand-alone law under an
arbitrary fixed momentum coefficient.  It is the result of balancing the three
leading terms in the upper envelope
\[
\frac{F(X_1)-F^*}{\gamma T}
+2\sqrt{1-\beta_1}\varsigma_{\max}S_0
+\frac{2\gamma\varsigma_{\max}^2\hat{\mathcal{L}}R_0}{1-\beta_1}.
\]
In particular, if one fixes $\delta:=1-\beta_1$ as in conventional training
(\emph{e.g.}, $\delta=0.1$ or $0.05$), then the same envelope becomes
\[
\frac{F(X_1)-F^*}{\gamma T}
+2\sqrt{\delta}\varsigma_{\max}S_0
+\frac{2\gamma\varsigma_{\max}^2\hat{\mathcal{L}}R_0}{\delta}.
\]
Optimizing only over $\gamma$ yields the different scale
$\gamma\asymp \sqrt{\delta/(R_0T)}$, and the stochastic term
$2\sqrt{\delta}\varsigma_{\max}S_0$ does not vanish with $T$.
Therefore, the rate and the learning-rate schedule in
Corollary~\ref{cor:muon_convergence} rely on the scheduled first-momentum gap
$1-\beta_1=\Theta(R_0^{1/2}/(S_0T^{1/2}))$.  The rank bound $R_0$ directly enters the
learning-rate scale because Muon's orthogonalized matrix direction satisfies
$\|O_t^{(l)}\|_{\mathsf{F}}\lesssim \sqrt{r^{(l)}}$, so the smoothness cost
accumulates as $\sum_l r^{(l)}$.  This rank-sensitive term is specific to the
matrix geometry of Muon and is not present in the same form for coordinate-wise
Adam or Euclidean SGD.  By contrast, the leading schedule for $1-\beta_1$
depends on the layer-wise noise aggregate $S_0$, the accumulated
rank $R_0$, and the horizon $T$; remaining problem-dependent effects enter
through constants such as $\hat{\mathcal{L}}$ and the initial optimality gap.

These rank factors belong to the nuclear-norm bound in
Corollary~\ref{cor:muon_convergence}. The finite-step Frobenius bound in
Corollary~\ref{cor:muon_dimension_independent} uses different spectral
estimates and needs no explicit rank factor in its schedule.

This theoretical schedule should not be confused with empirical neural scaling
laws of the Kaplan/Chinchilla type
\citep{Kaplan2020ScalingLaws,Hoffmann2022Chinchilla}.  Our analysis
characterizes how Muon's hyperparameters should depend on depth, accumulated
rank bound, and training horizon to guarantee stationarity; it does not by
itself imply a power-law relation between test loss, dataset size, model size,
and compute.
\end{remark}

\begin{remark}[Role of the Nesterov momentum correction]
\label{rem:nesterov_projection}
The Nesterov momentum correction plays a different role from the coupled
learning-rate and first-momentum schedules above.  When $\beta_2=1$,
$N_t^{(l)}$ reduces to the standard momentum $M_t^{(l)}$.  From
Theorem~\ref{thm:muon_convergence}, the Nesterov coefficient enters the initialization term and the following
noise--drift terms
\[
2\sqrt{2\beta_1(1-\beta_1\beta_2)(1-\beta_2)}
\varsigma_{\max}S_0
+\frac{2\gamma\beta_2\varsigma_{\max}^2\hat{\mathcal{L}}R_0}{1-\beta_1}.
\]
The bound therefore exhibits a trade-off: a smaller $\beta_2$ reduces the
initialization and drift coefficients, while changing the stochastic term.
This does not by itself prove an empirical early- or late-stage advantage.  The extra stochastic contribution and the change in the drift coefficient
relative to $\beta_2=1$ are lower-order under
$\beta_2=1-(1-\beta_1)^a$ with $a>\frac{1}{2}$; the drift term itself
remains leading-order. Thus the Nesterov correction does not change the leading $\mathcal{O}(T^{-\sfrac{1}{4}})$ convergence rate.
\end{remark}

\subsection{Dimension independence}
\label{sec:dimension_independence_discussion}
The bound in Corollary~\ref{cor:muon_dimension_independent} contains no
explicit matrix-dimension or rank factor when the block count and problem
constants are fixed. The constants $\ell_J$ and $u_J$ depend on the finite
polynomial map; dimension independence does not imply that these constants
are small or that the noise and smoothness constants stay fixed as an
architecture changes.

\begin{remark}[What the sharper spectral identities do and do not imply]
\label{rem:df_sharper_geometry}
Equations~\eqref{eq:spectral_AB}--\eqref{eq:full_shape_alignment} retain
the full finite map and give a useful descent-or-small-gradient test for
an individual block. However, $\|O\|_{\mathrm{op}}\le c_J$ cannot replace
$\|O\|_{\mathsf F}\le u_J$ in a Frobenius-smoothness bound without an
additional curvature condition. Similarly, an expected tracking-error
bound at each deterministic time does not imply the same bound at the
first time that alignment becomes unreliable. We therefore prove the
averaged stochastic result directly, using the same three-term
tracking decomposition, rather than relying on a stopping-time
conversion. Appendix~\ref{app:angular_scope} gives the precise geometric
certificate and explains this distinction. The present theorem establishes
dimension independence; it does not claim a uniformly small NS constant.
\end{remark}

\subsection{Implementation scope}
\label{sec:implementation_scope}

\begin{remark}[Generality beyond exact Newton--Schulz orthogonalization]
\label{rem:general_inexact_orthogonalization}
Our analysis does not require $O_t^{(l)}$ to be an exact orthogonal or polar
factor.  Exact orthogonalization corresponds to the special case
$\varsigma_i=1$ for all nonzero singular directions.  In contrast, the proof of
Theorem~\ref{thm:muon_convergence} only uses the two inequalities in
Proposition~\ref{prop:spectral_structure},
\[
\langle N_t^{(l)},O_t^{(l)}\rangle
\ge \varsigma_{\min}\|N_t^{(l)}\|_*,
\qquad
\|O_t^{(l)}\|_{\mathsf{F}}
\le \varsigma_{\max}\sqrt{r^{(l)}}.
\]
Therefore, exact polar decomposition is not needed.  It is sufficient that the
orthogonalization routine preserves enough directional alignment with
$N_t^{(l)}$ and keeps the update magnitude uniformly bounded.

This viewpoint is related to, but more algorithm-agnostic than, the recent
analysis of Muon with Newton--Schulz orthogonalization
\citep{KimOh2026MuonNewtonSchulz}.  That work studies the practical
Newton--Schulz iteration directly and shows that, for a fixed number of
Newton--Schulz steps, Muon converges at the same rate as the ideal SVD-polar
version up to a constant factor, with the factor converging to one rapidly as
the number of Newton--Schulz steps increases.  Our argument does not require
the modified singular values to be symmetric around $1$, nor does it require
$\varsigma_{\min}$ and $\varsigma_{\max}$ to approach $1$ as the number of
Newton--Schulz iterations grows.  The constants in our bound depend on
$\varsigma_{\min}$ and $\varsigma_{\max}$, and the same rate is preserved as
long as $\varsigma_{\min}>0$ and the distortion ratio
$\varsigma_{\max}/\varsigma_{\min}$ remains controlled.  Consequently, the
proof applies not only to Newton--Schulz approximations, but also to other
inexact orthogonalization or spectral-flattening procedures satisfying the same
alignment and magnitude conditions.
For Frobenius-normalized finite polynomial steps, however, $\phi_J(s)\to0$
as $s\to0$. Hence a uniform $\varsigma_{\min}>0$ need not exist over
arbitrary spectra. Sections~\ref{sec:finite_spectral_map} and
\ref{sec:dimension_independent_convergence} avoid this condition by using
$\phi_J(s)/s$ and proving Frobenius, rather than nuclear-norm, stationarity.
\end{remark}

\begin{remark}[Normalization and implementation scope]
\label{rem:df_normalization}
If the initial normalization is $N/(\kappa\|N\|_{\mathsf F})$ for a fixed
$\kappa\ge1$, apply the results to the effective map
$\widetilde\phi_J(s)=\phi_J(s/\kappa)$ on $[0,1]$. In particular, the
slope at zero is $(\prod_j a_j)/\kappa$, not $\prod_j a_j$.
For $N/(\|N\|_{\mathsf F}+\varepsilon)$ with fixed $\varepsilon>0$,
the same energy bound holds and
\[
\langle N,O\rangle\ge
\ell_J\frac{\|N\|_{\mathsf F}^2}{\|N\|_{\mathsf F}+\varepsilon}
\ge\ell_J(\|N\|_{\mathsf F}-\varepsilon).
\]
Accordingly, add $L\ell_J\varepsilon$ to the right-hand side of
\eqref{eq:df_convergence}; \eqref{eq:df_rate_explicit} acquires the
residual $2L\varepsilon$.

If the output in block $l$ is multiplied by a fixed $\alpha_l>0$, the same
proof applies with
$\ell_J$ replaced by $\alpha_{\min}\ell_J$ and
$u_J$ by $\alpha_{\max}u_J$, where
$\alpha_{\min}=\min_l\alpha_l$ and
$\alpha_{\max}=\max_l\alpha_l$; the operator-norm bound becomes
$\alpha_{\max}c_J$. This follows by applying the blockwise inequalities
with these common bounds. Such scaling preserves dimension independence
only when $\alpha_{\max}$ and $1/\alpha_{\min}$ remain bounded
independently of matrix size. Unbounded size-dependent multipliers cannot
be absorbed into a dimension-independent constant. No finite-precision
error guarantee is asserted here.
\end{remark}


\section{Conclusion}
\label{sec:conclusion}
We established convergence guarantees for layer-wise Muon with finite
Newton--Schulz updates and Nesterov momentum under joint
gradient-dependent smoothness and bounded conditional variance.
A common descent-and-tracking argument separates the momentum error into
initialization, stochastic noise, and gradient drift. Controlling the
finite polynomial through $\phi_J(s)/s$ yields an
$\mathcal O(T^{-1/4})$ bound on the expected average Frobenius gradient norm
under coupled learning-rate and momentum schedules, including the standard
equal-coefficient Nesterov rule. The bound has no explicit matrix-dimension
or rank factors when the block count and problem constants are fixed.
We also verified the scalar-map conditions analytically for the original
five-step quintic iteration.

The central implication is that finite orthogonalization provides sufficient
alignment and update-energy control for convergence without a positive
lower bound on its nonzero output singular values. The complementary
nuclear-norm result identifies where rank dependence enters under stronger
spectral control. Extending the analysis to finite-precision arithmetic and
to fixed momentum with increasing batch sizes would further connect these
guarantees to training practice.

\bibliography{./refs}
\bibliographystyle{apalike}

\clearpage
\appendix

\section{Useful Lemmas}

\begin{lemma}
\label{lemma_smooth}
Under Assumption~\ref{assump:layer_wise_smoothness}, for any $X,Y\in\mathcal S$,
\begin{equation}
F(Y)\le F(X)+\langle\nabla F(X),Y-X\rangle
+\frac{\mathcal L_0+\mathcal L_1\|\nabla F(X)\|_{\mathsf F}^{q}}{2}
\sum_{l=1}^L\|Y^{(l)}-X^{(l)}\|_{\mathsf F}^{2}.
\end{equation}
\end{lemma}
\begin{proof}
Writing $V=Y-X$ and applying the fundamental theorem of calculus gives
\begin{equation}
\begin{aligned}
F(Y)-F(X)-\langle\nabla F(X),V\rangle
&=\int_0^1\langle\nabla F(X+tV)-\nabla F(X),V\rangle\,dt\\
&\le\int_0^1\|\nabla F(X+tV)-\nabla F(X)\|_{\mathsf F}\|V\|_{\mathsf F}\,dt\\
&\le\frac{\mathcal L_0+\mathcal L_1\|\nabla F(X)\|_{\mathsf F}^{q}}{2}\|V\|_{\mathsf F}^2.
\end{aligned}
\end{equation}
The last identity $\|V\|_{\mathsf F}^2=\sum_l\|V^{(l)}\|_{\mathsf F}^2$ proves the claim.
\end{proof}

\begin{lemma}
Let $\{N_t^{(l)}\}_{t=1}^T$ be the generalized Nesterov momentum sequence generated by Algorithm~\ref{alg:muon_layerwise}. Then, for $t\ge2$, $N_t^{(l)}$ admits the following equivalent recursive form (with $N_1^{(l)}=(1-\beta_1\beta_2)G_1^{(l)}$):
\begin{equation}
N_t^{(l)} =\beta_1 N_{t-1}^{(l)} + (1-\beta_1 \beta_2) G_{t}^{(l)} - \beta_1 (1-\beta_2) G_{t-1}^{(l)} = \beta_1 N_{t-1}^{(l)} + (1-\beta_1) G_t^{(l)} + \beta_1(1-\beta_2)(G_t^{(l)} - G_{t-1}^{(l)}).
\end{equation}
\end{lemma}

\begin{proof}
By the update rule for $N_t^{(l)}$, we obtain
\begin{equation}
\begin{aligned}
N_t^{(l)} - \beta_1 N_{t-1}^{(l)} = & \beta_2 M_{t}^{(l)} + (1-\beta_2) G_{t}^{(l)} - \beta_1 (\beta_2 M_{t-1}^{(l)} + (1-\beta_2) G_{t-1}^{(l)}) \\ 
=& \beta_2 (\beta_1 M_{t-1}^{(l)} + (1-\beta_1) G_{t}^{(l)}) + (1-\beta_2) G_{t}^{(l)} - \beta_1 \beta_2 M_{t-1}^{(l)} - \beta_1 (1-\beta_2) G_{t-1}^{(l)} \\
=& (1-\beta_1 \beta_2) G_{t}^{(l)} - \beta_1 (1-\beta_2) G_{t-1}^{(l)} \\
=& (1-\beta_1) G_{t}^{(l)} + \beta_1 (1-\beta_2)(G_t^{(l)} - G_{t-1}^{(l)}).
\end{aligned}
\end{equation}
Rearranging terms gives the desired recursion.
\end{proof}

\begin{lemma}
Let $a,b>0$ and $0<\alpha<\beta$. Define
\[
\bar{x}:=\left(a+b^{\sfrac{\alpha}{\beta}}\right)^{\sfrac{1}{\alpha}}.
\]
Then, for every $x\ge \bar{x}$,
\[
\frac{a}{x^\alpha}+\frac{b}{x^\beta}\le 1.
\]
Moreover, if $x_\star$ denotes the smallest positive solution of
$a/x^\alpha+b/x^\beta\le 1$, then
\[
x_\star \le \bar{x} \le 2^{\sfrac{1}{\alpha}}x_\star .
\]
Equivalently, the corresponding threshold in the variable $s=x^\alpha$ is within a factor of $2$ of the optimal threshold.
\label{lemma_4}
\end{lemma}

\begin{proof}
Let $s=x^\alpha$ and $\gamma=\beta/\alpha>1$. The desired inequality is equivalent to
\begin{equation}
\frac{a}{s} + \frac{b}{s^{\gamma}} \le 1
\label{S.18}
\end{equation}
with $s>0$. Any feasible $s$ for \eqref{S.18} must satisfy
\begin{equation}
s \ge \max\{a,b^{\sfrac{1}{\gamma}}\},
\label{S.19}
\end{equation}
since each term on the left-hand side of \eqref{S.18} is nonnegative.

Now set $\bar{s}:=a+b^{\sfrac{1}{\gamma}}$. Since $\bar{s}\ge b^{\sfrac{1}{\gamma}}$, we have
$b/\bar{s}^{\gamma}\le b^{\sfrac{1}{\gamma}}/\bar{s}$. Therefore,
\begin{equation}
\frac{a}{\bar{s}}+\frac{b}{\bar{s}^{\gamma}}
\le
\frac{a}{\bar{s}}+\frac{b^{\sfrac{1}{\gamma}}}{\bar{s}}
=1.
\end{equation}
Thus $\bar{s}$ is feasible, and consequently every $x\ge \bar{x}:=\bar{s}^{\sfrac{1}{\alpha}}$
satisfies the claimed inequality.

Let $s_\star$ be the smallest feasible value of $s$ in \eqref{S.18}. By \eqref{S.19},
$s_\star\ge \max\{a,b^{\sfrac{1}{\gamma}}\}$. Hence
\[
\bar{s}=a+b^{\sfrac{1}{\gamma}}
\le 2\max\{a,b^{\sfrac{1}{\gamma}}\}
\le 2s_\star .
\]
Taking the power $1/\alpha$ gives $\bar{x}\le 2^{\sfrac{1}{\alpha}}x_\star$, where
$x_\star=s_\star^{\sfrac{1}{\alpha}}$. The inequality $x_\star\le \bar{x}$ follows from the feasibility of $\bar{x}$.
\end{proof}

\section{Proof of Theorem~\ref{thm:muon_convergence}}
\label{app:rank_sensitive_convergence}
\begin{proof}
\noindent\textbf{Step 1: Descent and telescoping.}
Applying Lemma~\ref{lemma_smooth} with $Y=X_{t+1}$ and $X=X_t$ gives
\begin{equation}
F(X_{t+1}) \le F(X_t) + \langle \nabla F(X_t), X_{t+1} - X_t \rangle + \sum_{l=1}^L \frac{\mathcal{L}_0 + \mathcal{L}_1 \|\nabla F(X_t)\|_{\mathsf{F}}^q}{2} \|X_{t+1}^{(l)} - X_t^{(l)}\|_{\mathsf{F}}^2.
\end{equation}
For compactness in this proof, write
$K_t=\mathcal L_0+\mathcal L_1\|\nabla F(X_t)\|_{\mathsf F}^q$.
Substituting the update rule $X_{t+1}^{(l)}=X_t^{(l)}-\gamma O_t^{(l)}$ into the preceding inequality yields
\begin{equation}
\begin{aligned}
F(X_{t+1})
&\le F(X_t)+\langle\nabla F(X_t),X_{t+1}-X_t\rangle\\
&\quad+\frac{K_t}{2}\sum_{l=1}^L
  \|X_{t+1}^{(l)}-X_t^{(l)}\|_{\mathsf F}^2\\
&=F(X_t)-\gamma\sum_{l=1}^L
  \langle\nabla F_{(l)}(X_t),O_t^{(l)}\rangle\\
&\quad+\frac{\gamma^2K_t}{2}\sum_{l=1}^L\|O_t^{(l)}\|_{\mathsf F}^2\\
&=F(X_t)-\gamma\sum_{l=1}^L\langle N_t^{(l)},O_t^{(l)}\rangle\\
&\quad+\gamma\sum_{l=1}^L
  \langle N_t^{(l)}-\nabla F_{(l)}(X_t),O_t^{(l)}\rangle\\
&\quad+\frac{\gamma^2K_t}{2}\sum_{l=1}^L\|O_t^{(l)}\|_{\mathsf F}^2\\
&\le F(X_t)-\gamma\sum_{l=1}^L\langle N_t^{(l)},O_t^{(l)}\rangle\\
&\quad+\gamma\sum_{l=1}^L
  \|N_t^{(l)}-\nabla F_{(l)}(X_t)\|_{\mathsf F}
  \|O_t^{(l)}\|_{\mathsf F}\\
&\quad+\frac{\gamma^2K_t}{2}\sum_{l=1}^L\|O_t^{(l)}\|_{\mathsf F}^2.
\end{aligned}
\end{equation}

By Proposition~\ref{prop:spectral_structure}, the inexact orthogonalization direction satisfies $\langle N_t^{(l)}, O_t^{(l)} \rangle \ge \varsigma_{\min} \| N_t^{(l)} \|_*$ and $\| O_t^{(l)} \|_{\mathsf{F}} \le  \varsigma_{\max} \sqrt{r^{(l)}}$.
Write $D_t^{(l)}=N_t^{(l)}-\nabla F_{(l)}(X_t)$ and
$H_t=\mathcal L_0+(1-q)\mathcal L_1
  +q\mathcal L_1\|\nabla F(X_t)\|_{\mathsf F}$. Then
\begin{equation}
\begin{aligned}
F(X_{t+1})
&\le F(X_t)-\gamma\varsigma_{\min}\sum_{l=1}^L\|N_t^{(l)}\|_*
 +\gamma\varsigma_{\max}\sum_{l=1}^L\sqrt{r^{(l)}}\|D_t^{(l)}\|_{\mathsf F}\\
&\quad+\frac{\gamma^2\varsigma_{\max}^2K_t}{2}\sum_{l=1}^L r^{(l)}\\
&\overset{(i)}{\le}F(X_t)
 -\gamma\varsigma_{\min}\sum_{l=1}^L\|N_t^{(l)}\|_*
 +\gamma\varsigma_{\max}\sum_{l=1}^L\sqrt{r^{(l)}}\|D_t^{(l)}\|_{\mathsf F}\\
&\quad+\frac{\gamma^2\varsigma_{\max}^2H_t}{2}\sum_{l=1}^L r^{(l)}\\
&\overset{(ii)}{\le}F(X_t)
 -\gamma\varsigma_{\min}\sum_{l=1}^L\|\nabla F_{(l)}(X_t)\|_*\\
&\quad+\gamma\varsigma_{\min}\sum_{l=1}^L\|D_t^{(l)}\|_*
 +\gamma\varsigma_{\max}\sum_{l=1}^L\sqrt{r^{(l)}}\|D_t^{(l)}\|_{\mathsf F}\\
&\quad+\frac{\gamma^2\varsigma_{\max}^2H_t}{2}\sum_{l=1}^L r^{(l)}\\
&\overset{(iii)}{\le}F(X_t)
 -\gamma\varsigma_{\min}\sum_{l=1}^L\|\nabla F_{(l)}(X_t)\|_*\\
&\quad+2\gamma\varsigma_{\max}\sum_{l=1}^L\sqrt{r^{(l)}}\|D_t^{(l)}\|_{\mathsf F}
 +\frac{\gamma^2\varsigma_{\max}^2H_t}{2}\sum_{l=1}^L r^{(l)}.
\end{aligned}
\end{equation}
Here, $(i)$ follows from $a^q \le (1-q)+qa$ for $a\ge 0$ and $q\in[0,1]$; $(ii)$ follows from the reverse triangle inequality $\|A\|_* \ge \|B\|_*-\|B-A\|_*$; and $(iii)$ uses $\|A\|_*\le \sqrt{r}\|A\|_{\mathsf{F}} $where $r$ bounds the rank of the error matrix $A=\nabla_{(l)}F(X_t)-N_t^{(l)}$, as stipulated in Theorem~\ref{thm:muon_convergence}, and $\varsigma_{\min} \le \varsigma_{\max}$.

Taking expectations and summing the resulting descent inequality over $t=1,\ldots,T$, the intermediate terms telescope and give
\begin{equation}
\begin{aligned}
\mathbb E[F(X_{T+1})]
&\le F(X_1)-\gamma\varsigma_{\min}
 \sum_{t=1}^T\sum_{l=1}^L
 \mathbb E\|\nabla F_{(l)}(X_t)\|_*\\
&\quad+2\gamma\varsigma_{\max}
 \sum_{t=1}^T\sum_{l=1}^L\sqrt{r^{(l)}}
 \mathbb E\|D_t^{(l)}\|_{\mathsf F}\\
&\quad+\frac{\gamma^2\varsigma_{\max}^2}{2}
 \sum_{t=1}^T\sum_{l=1}^L r^{(l)}\mathbb E[H_t].
\end{aligned}
\end{equation}

Rearranging the terms and using Assumption~\ref{assump:lower_boundedness}, we obtain

\begin{equation}
\begin{aligned}
&\frac{\varsigma_{\min}}{T}\sum_{t=1}^{T}\sum_{l=1}^L{\mathbb E}\left[\left\Vert \nabla F_{(l)}(X_t) \right\Vert_*\right] - \frac{\gamma \varsigma_{\max}^2 q}{2T}\sum_{t=1}^{T}\sum_{l=1}^L r^{(l)}\mathcal{L}_1{\mathbb E}\left[\left\Vert\nabla F(X_t)\right\Vert_{\mathsf{F}}\right] \\
\le & \frac{F(X_{1}) - F^*}{\gamma T}
 + \frac{2\varsigma_{\max}}{T}\sum_{t=1}^{T}\sum_{l=1}^L \sqrt{r^{(l)}} {\mathbb E}\left[\left\Vert N_t^{(l)}-\nabla F_{(l)}(X_t)\right\Vert_{\mathsf{F}}\right]
+ \frac{\gamma \varsigma_{\max}^2}{2}\sum_{l=1}^L r^{(l)}\hat{\mathcal{L}},
\end{aligned}
\label{Eq.s_9}
\end{equation}

where $\hat{\mathcal{L}} := \mathcal{L}_0 + (1-q)\mathcal{L}_1$.

\noindent\textbf{Step 2: The Nesterov tracking-error expansion.}
We next control the discrepancy between the extrapolated momentum and the corresponding true gradient block. From the recursion
$N_t^{(l)}=\beta_1 N_{t-1}^{(l)}+(1-\beta_1\beta_2)G_t^{(l)}-\beta_1(1-\beta_2)G_{t-1}^{(l)}$, we have
\begin{equation}
\begin{aligned}
N_t^{(l)}-\nabla F_{(l)}(X_t)  = & \beta_1 \left( N_{t-1}^{(l)} - \nabla F_{(l)}(X_{t-1}) \right) + (1-\beta_1 \beta_2) \left( G_t^{(l)} -  \nabla F_{(l)}(X_t) \right) \\
& - \beta_1(1-\beta_2)  \left( G_{t-1}^{(l)} - \nabla F_{(l)}(X_{t-1}) \right) + \beta_1\beta_2 \left( \nabla F_{(l)}(X_{t-1}) - \nabla F_{(l)}(X_t) \right) \\
\end{aligned}
\end{equation}

Unrolling this recursion, together with the initialization $M_0^{(l)}=0$, gives
\begin{equation}
\label{eq:tracking_expansion}
\begin{aligned}
N_t^{(l)}-\nabla F_{(l)}(X_t) =& -\beta_1^t\beta_2\nabla F_{(l)}(X_{1}) + (1-\beta_1 \beta_2)(G_t^{(l)} - \nabla F_{(l)}(X_t)) \\
&+ \beta_2(1-\beta_1) \sum_{k=1}^{t-1} \beta_1^{t-k}(G_{k}^{(l)} - \nabla F_{(l)}(X_{k})) \\
&+\beta_1\beta_2\sum_{k=2}^t \beta_1^{t-k} (\nabla F_{(l)}(X_{k-1}) - \nabla F_{(l)}(X_k)),
\end{aligned} 
\end{equation}

Taking norms, expectations, and then averaging over $t=1,\ldots,T$, we decompose the error into three terms:
Write $\varepsilon_t^{(l)}=G_t^{(l)}-\nabla F_{(l)}(X_t)$.
\begin{equation}
\label{Eq.s_11}
\frac1T\sum_{t=1}^T\mathbb E\|D_t^{(l)}\|_{\mathsf F}
\le {\cal T}_1+{\cal T}_2+{\cal T}_3,
\end{equation}
where the initialization, noise, and drift terms are, respectively,
\begin{align*}
{\cal T}_1
&=\frac{\beta_2}{T}\sum_{t=1}^T\beta_1^t
  \|\nabla F_{(l)}(X_1)\|_{\mathsf F},\\
{\cal T}_2
&=\frac1T\sum_{t=1}^T\mathbb E\Bigl\|
 (1-\beta_1\beta_2)\varepsilon_t^{(l)}\\
&\hspace{4em}+\beta_2(1-\beta_1)
 \sum_{k=1}^{t-1}\beta_1^{t-k}\varepsilon_k^{(l)}
 \Bigr\|_{\mathsf F},\\
{\cal T}_3
&=\frac1T\sum_{t=1}^T\mathbb E\Bigl\|
 \beta_1\beta_2\sum_{k=2}^t\beta_1^{t-k}\\
&\hspace{4em}\cdot
 \bigl(\nabla F_{(l)}(X_{k-1})-\nabla F_{(l)}(X_k)\bigr)
 \Bigr\|_{\mathsf F}.
\end{align*}

\noindent\textbf{Step 3: Initialization, noise, and gradient drift.}
The first term is controlled directly by the geometric series:
\begin{equation}
{\cal T}_1 = \frac{\beta_2}{T}\sum_{t=1}^{T} \beta_1^t\left\Vert \nabla F_{(l)}(X_{1})\right\Vert_{\mathsf{F}} \le \frac{\beta_2\left\Vert \nabla F_{(l)}(X_{1})\right\Vert_{\mathsf{F}}}{T(1-\beta_1)}.
\end{equation}

For the stochastic-gradient noise term ${\cal T}_2$, Jensen's inequality and the martingale-difference property of the stochastic gradients yield
\begin{equation}
\begin{aligned}
{\cal T}_2
&=\frac1T\sum_{t=1}^T\mathbb E\Bigl\|
 (1-\beta_1\beta_2)\varepsilon_t^{(l)}
 +(1-\beta_1)\beta_2\sum_{k=1}^{t-1}\beta_1^{t-k}\varepsilon_k^{(l)}
 \Bigr\|_{\mathsf F}\\
&\overset{(i)}{\le}\frac1T\sum_{t=1}^T
 \sqrt{\mathbb E\Bigl\|
 (1-\beta_1\beta_2)\varepsilon_t^{(l)}
 +(1-\beta_1)\beta_2\sum_{k=1}^{t-1}\beta_1^{t-k}\varepsilon_k^{(l)}
 \Bigr\|_{\mathsf F}^2}\\
&\overset{(ii)}{=}\frac1T\sum_{t=1}^T
 \Biggl((1-\beta_1\beta_2)^2\mathbb E\|\varepsilon_t^{(l)}\|_{\mathsf F}^2\\
&\hspace{6em}+(1-\beta_1)^2\beta_2^2
 \sum_{k=1}^{t-1}\beta_1^{2(t-k)}
 \mathbb E\|\varepsilon_k^{(l)}\|_{\mathsf F}^2\Biggr)^{1/2}\\
&\overset{(iii)}{\le}\frac1T\sum_{t=1}^T
 \sqrt{(1-\beta_1\beta_2)^2
 +(1-\beta_1)^2\beta_2^2\sum_{k=1}^{t-1}\beta_1^{2(t-k)}}\,\sigma_{(l)}\\
&\le\sqrt{(1-\beta_1\beta_2)^2
 +\frac{(1-\beta_1)^2\beta_1^2\beta_2^2}{1-\beta_1^2}}\,\sigma_{(l)}\\
&=\sqrt{\frac{1-\beta_1
 +2\beta_1(1-\beta_1\beta_2)(1-\beta_2)}{1+\beta_1}}\,\sigma_{(l)}\\
&\overset{(iv)}{\le}
 \left(\sqrt{1-\beta_1}
 +\sqrt{2\beta_1(1-\beta_1\beta_2)(1-\beta_2)}\right)\sigma_{(l)}.
\end{aligned}
\end{equation}

Here, $(i)$ is Jensen's inequality in the form $\mathbb{E}[Z]\le \sqrt{\mathbb{E}[Z^2]}$; $(ii)$ uses conditional unbiasedness $\mathbb{E}[G_k^{(l)}-\nabla F_{(l)}(X_k)\mid\mathcal F_{k-1}]=0$ from Assumption~\ref{assump:layer_wise_variance}, which eliminates the cross terms; $(iii)$ applies the variance bound in Assumption~\ref{assump:layer_wise_variance}; and $(iv)$ follows from $\sqrt{a+b}\le \sqrt{a}+\sqrt{b}$.

It remains to bound the drift term ${\cal T}_3$, which is caused by the change of the true gradient along the trajectory. Write ${\cal T}_{3,l}$ for this term in layer $l$. Cauchy--Schwarz across layers and the triangle inequality give
\begin{equation}
\begin{aligned}
\sum_{l=1}^L\sqrt{r^{(l)}}\,{\cal T}_{3,l}
&\le\frac{\beta_1\beta_2\sqrt{R_0}}{T}
\sum_{t=1}^T\sum_{k=2}^t\beta_1^{t-k}
\mathbb E\|\nabla F(X_{k-1})-\nabla F(X_k)\|_{\mathsf F}\\
&\le\frac{\gamma\beta_1\beta_2\varsigma_{\max}R_0}{T}
\sum_{t=1}^T\sum_{k=2}^t\beta_1^{t-k}
\mathbb E\bigl[\mathcal L_0+\mathcal L_1\|\nabla F(X_k)\|_{\mathsf F}^q\bigr]\\
&\le\frac{\gamma\beta_2\varsigma_{\max}R_0}{1-\beta_1}
\left(\hat{\mathcal L}+
\frac{q\mathcal L_1}{T}\sum_{k=1}^T\mathbb E\|\nabla F(X_k)\|_{\mathsf F}\right).
\end{aligned}
\label{Eq.s_14}
\end{equation}
The second inequality applies Assumption~\ref{assump:layer_wise_smoothness} with base point $X_k$ and uses
$X_k-X_{k-1}=-\gamma O_{k-1}$ and $\|O_{k-1}\|_{\mathsf F}\le\varsigma_{\max}\sqrt{R_0}$.
The last inequality uses $z^q\le(1-q)+qz$, exchanges the finite sums, and bounds the geometric series by $(1-\beta_1)^{-1}$. In particular, no one-layer bound is applied to a displacement in all layers.

Combining the bounds for ${\cal T}_1$, ${\cal T}_2$, and ${\cal T}_3$ in Eq.~\eqref{Eq.s_11}, we obtain
\begin{equation}
\begin{aligned}
\frac1T\sum_{t=1}^T\sum_{l=1}^L\sqrt{r^{(l)}}\,
\mathbb E\|N_t^{(l)}-\nabla_{(l)}F(X_t)\|_{\mathsf F}
&\le\frac{\beta_2G_0}{T(1-\beta_1)}
+\left(\sqrt{1-\beta_1}+\sqrt{2\beta_1(1-\beta_1\beta_2)(1-\beta_2)}\right)S_0\\
&\quad+\frac{\gamma\beta_2\varsigma_{\max}R_0\hat{\mathcal L}}{1-\beta_1}
+\frac{\gamma\beta_2\varsigma_{\max}R_0q\mathcal L_1}{1-\beta_1}
\frac1T\sum_{t=1}^T\mathbb E\|\nabla F(X_t)\|_{\mathsf F}.
\end{aligned}
\label{Eq.s_15}
\end{equation}

\noindent\textbf{Step 4: Substitution and absorption.}
Finally, substituting Eq.~\eqref{Eq.s_15} into Eq.~\eqref{Eq.s_9} and using $\|\nabla F(X_t)\|_{\mathsf F}\le\sum_l\|\nabla_{(l)}F(X_t)\|_*$ and $\beta_2\le1$ gives
\begin{equation}
\begin{aligned}
&\left(\varsigma_{\min} - \frac{\gamma \varsigma_{\max}^2q\mathcal{L}_1R_0}{2} - \frac{2\gamma \varsigma_{\max}^2q\mathcal{L}_1R_0}{1-\beta_1} \right)\frac{1}{T}\sum_{t=1}^{T}\sum_{l=1}^L{\mathbb E}\left[\left\Vert \nabla F_{(l)}(X_t) \right\Vert_*\right] \\
\le & \frac{F(X_{1}) - F^*}{\gamma T}
 + \frac{2 \beta_2\varsigma_{\max}}{T(1-\beta_1)}\sum_{l=1}^L \sqrt{r^{(l)}} \left\Vert \nabla F_{(l)}(X_{1})\right\Vert_{\mathsf{F}}\\
  &+ {2\left({\sqrt{1-\beta_1}} + {\sqrt{2\beta_1(1-\beta_1\beta_2)(1-\beta_2)}} \right)\varsigma_{\max}}\sum_{l=1}^L\sqrt{{r^{(l)}}}\sigma_{(l)} \\
&+ \frac{2\gamma\beta_2\varsigma_{\max}^2\hat{\mathcal{L}}\sum_{l=1}^L {r^{(l)}}}{1-\beta_1}
+ \frac{\gamma \varsigma_{\max}^2\hat{\mathcal{L}}\sum_{l=1}^L r^{(l)}}{2},
\end{aligned}
\label{Eq.s_17}
\end{equation}
where $R_0=\sum_{l=1}^L r^{(l)}$. This is precisely the claimed convergence bound. The use of $R_0$ in the absorption coefficient accounts for simultaneous, coupled block updates.
\end{proof}

\section{Proof of Corollary~\ref{cor:muon_convergence}}

\begin{proof}
  Applying $\beta_2 = 1-(1-\beta_1)^a$ for $a > \frac{1}{2}$, we have
\begin{equation}
  \sqrt{\beta_1(1-\beta_1\beta_2)(1-\beta_2)}
  = \sqrt{\beta_1(1-\beta_1 + \beta_1(1-\beta_1)^a)(1-\beta_1)^a}
  =\sqrt{\beta_1(1-\beta_1)^{1+a}+\beta_1^2(1-\beta_1)^{2a}}.
\end{equation}

We need the term
$\sqrt{\beta_1(1-\beta_1)^{1+a}+\beta_1^2(1-\beta_1)^{2a}}$
to be of no larger order than $\sqrt{1-\beta_1}$ so that the Nesterov
momentum projection in Muon does not affect the final convergence rate. This
is equivalent to
$\beta_1(1-\beta_1)^{1+a}+\beta_1^2(1-\beta_1)^{2a}=O(1-\beta_1)$, which
holds for $a>\frac{1}{2}$.

Since $\beta_2\le 1$, we minimize the following upper envelope of the three
leading terms in Eq.~\eqref{Eq.s_17}. By Young's inequality,
\begin{equation}
\begin{aligned}
& \frac{F(X_1)-F^*}{\gamma T}
+ 2\sqrt{1-\beta_1}\varsigma_{\max}S_0
+ \frac{2\gamma\varsigma_{\max}^2\hat{\mathcal{L}}R_0}{1-\beta_1}  \\
\ge& \left(\frac{4 (F(X_1)-F^*)}{\gamma T}\right)^{\sfrac{1}{4}}\cdot
\left(4\sqrt{1-\beta_1}\varsigma_{\max}S_0\right)^{\sfrac{1}{2}}\cdot
\left( \frac{8\gamma\varsigma_{\max}^2 \hat{\mathcal{L}}R_0}{1-\beta_1}\right)^{\sfrac{1}{4}} \\
= &  \frac{512^{\sfrac{1}{4}}\varsigma_{\max}\hat{\mathcal{L}}^{\sfrac{1}{4}}(F(X_1)-F^*)^{\sfrac{1}{4}}S_0^{\sfrac{1}{2}}R_0^{\sfrac{1}{4}}}{T^{\sfrac{1}{4}}}.
\end{aligned}
\end{equation}
The equality condition is
\begin{equation}
\frac{4(F(X_1)-F^*)}{\gamma T}
= 4\sqrt{1-\beta_1}\varsigma_{\max}S_0
= \frac{8\gamma\varsigma_{\max}^2\hat{\mathcal{L}}R_0}{1-\beta_1}.
\end{equation}
Solving these two balancing equations gives
\begin{equation}
\gamma =
\frac{(F(X_1)-F^*)^{\sfrac{3}{4}}}
{2^{\sfrac{1}{4}}\varsigma_{\max}S_0^{\sfrac{1}{2}}\hat{\mathcal{L}}^{\sfrac{1}{4}}R_0^{\sfrac{1}{4}}T^{\sfrac{3}{4}}},
\quad \beta_1 =
1-\frac{(2\hat{\mathcal{L}}(F(X_1)-F^*)R_0)^{\sfrac{1}{2}}}{S_0T^{\sfrac{1}{2}}}.
\end{equation}
At this equality point, the minimized upper envelope is
\begin{equation}
\frac{512^{\sfrac{1}{4}}\varsigma_{\max}S_0^{\sfrac{1}{2}}R_0^{\sfrac{1}{4}}
\hat{\mathcal{L}}^{\sfrac{1}{4}}(F(X_1)-F^*)^{\sfrac{1}{4}}
}{T^{\sfrac{1}{4}}}.
\end{equation}

Define $C_1,C_2$ as in Corollary~\ref{cor:muon_convergence}. Choosing the horizon in \eqref{eq:original_horizon}, Lemma~\ref{lemma_4} (or a direct check if $q\mathcal L_1=0$) gives
\begin{equation}
\frac{\gamma \varsigma_{\max}^2q\mathcal{L}_1R_0}{2}
+ \frac{2\gamma \varsigma_{\max}^2q\mathcal{L}_1R_0}{1-\beta_1}
\le \frac{\varsigma_{\min}}{2}.
\end{equation}

Then, we reformulate Eq. (\ref{eq:muon_convergence}) as

\begin{equation}
\begin{aligned}
\frac{1}{T}\sum_{t=1}^{T}\sum_{l=1}^L
\mathbb E [\Vert\nabla F_{(l)}(X_t)\Vert_*]
\le& \frac{2}{\varsigma_{\min}} \Bigg(
\frac{512^{\sfrac{1}{4}}\varsigma_{\max}S_0^{\sfrac{1}{2}}R_0^{\sfrac{1}{4}}
\hat{\mathcal{L}}^{\sfrac{1}{4}}(F(X_1)-F^*)^{\sfrac{1}{4}}
}{T^{\sfrac{1}{4}}} \\
&\quad
+ \frac{2\varsigma_{\max}G_0}{C_2T^{\sfrac{1}{2}}} \\
&\quad
+ 2\sqrt{2}\varsigma_{\max}S_0
\left(\frac{ C_2^{\sfrac{(1+a)}{2}}}{T^{\sfrac{(1+a)}{4}}}
+\frac{ C_2^{a}}{T^{\sfrac{a}{2}}}\right) \\
&\quad
+ \frac{C_1\varsigma_{\max}^2\hat{\mathcal{L}}
R_0}{2T^{\sfrac{3}{4}}}
\Bigg).
\end{aligned}
\end{equation}
\end{proof}

\section{Proofs of the dimension-independent results}
\label{app:dimension_independence}

\subsection{Proof of Proposition~\ref{prop:finite_map_geometry}}
\begin{proof}
Each polynomial step preserves the representation in the original left and
right singular-vector bases and applies $p_j$ to its signed diagonal entries.
The final entries are positive by \eqref{eq:finite_ns_constants}, giving the
displayed SVD formula for $O$.
For $s_i>0$, put $r_i=\phi_J(s_i)/s_i$ and $w_i=s_i^2$, so
$\sum_iw_i=1$ and $r_i\in[\ell_J,u_J]$. Direct calculation gives
\begin{equation}
\label{eq:df_AB_weights}
A=\sum_iw_ir_i,\qquad B=\sum_iw_ir_i^2,
\qquad \langle N,O\rangle=\|N\|_{\mathsf F}A.
\end{equation}
The energy estimates follow from $\ell_J\le r_i\le u_J$ and
$r_i^2\le u_Jr_i$. The operator-norm estimate follows from the largest
output singular value.

For the angular estimate, $(r_i-\ell_J)(u_J-r_i)\ge0$ implies
\[
B\le(u_J+\ell_J)A-u_J\ell_J.
\]
Consequently,
\begin{equation}
\label{eq:df_kantorovich_proof}
\frac{B}{A^2}\le
\frac{u_J+\ell_J}{A}-\frac{u_J\ell_J}{A^2}
\le\frac{(u_J+\ell_J)^2}{4u_J\ell_J}.
\end{equation}
The last expression is the maximum over $A>0$, attained at
$A=2u_J\ell_J/(u_J+\ell_J)$. This is the Kantorovich angle bound,
derived here directly. Finally,
\[
\langle H,O\rangle
=\langle N,O\rangle-\langle E,O\rangle
\ge\|N\|_{\mathsf F}A-\|E\|_{\mathsf F}\sqrt B
\ge A(\|N\|_{\mathsf F}-\bar\rho_J\|E\|_{\mathsf F}).
\]
The estimates needed by the convergence theorem also hold for $N=0$,
where $O=0$.
\end{proof}

\subsection{Analytic constants for the fixed quintic}
\label{app:explicit_ns_constants}
Let $p(s)=as+bs^3+cs^5$, with the exact decimal coefficients in
\eqref{eq:explicit_ns_constants}, and let $h(z)=a+bz+cz^2$.
Completing the square gives
\begin{equation}
\label{eq:quintic_positive_gain}
h(z)\ge a-\frac{b^2}{4c}
=0.638614570514\ldots>0,\qquad z\ge0.
\end{equation}
For completeness, the two positive stationary points of $p$ are
$\sqrt{z_-}$ and $\sqrt{z_+}$, where
\[
z_\pm=\frac{-3b\pm\sqrt{9b^2-20ac}}{10c}.
\]
Evaluating $p$ at these points and at the endpoints proves the interval
inclusions
\begin{equation}
\label{eq:ns_invariant_intervals}
p([0,5/4])\subseteq[0,5/4],\qquad
p([17/25,5/4])\subseteq[17/25,5/4].
\end{equation}
For example, the local maximum is between $1.2023$ and $1.2024$, the local
minimum is between $0.6818$ and $0.6819$,
$p(17/25)=1.1362138043392$, and $p(5/4)=1.17909912109375$.
These loose rational brackets suffice for \eqref{eq:ns_invariant_intervals};
no optimization of a high-degree composite polynomial is needed.
Also, $h$ is decreasing on $[0,(17/25)^2]$ and
$h((17/25)^2)>1$, so $p(s)\ge s$ for $0\le s\le17/25$.

Consider an initial $s\in(0,1]$. While an iterate is below $17/25$, it
cannot decrease. Once it reaches $[17/25,5/4]$, it remains there by
\eqref{eq:ns_invariant_intervals}. If $s\ge17/25$, its later iterates
are at least $17/25\ge(17/25)s$; if $s<17/25$, they are at least $s$.
Thus $\phi_J(s)\ge(17/25)s$ for every $J\ge1$.
All iterates lie in $[0,5/4]$, and on this interval
$h(s^2)\le a$, because $(5/4)^2<-b/c$. It follows that
\[
\frac{\phi_J(s)}s=\prod_{j=0}^{J-1}h(\phi_j(s)^2)\le a^J.
\]
The limit at $s\downarrow0$ is $a^J$, proving $u_J=a^J$.
This proves all of \eqref{eq:explicit_ns_constants}.
The general theorem permits different coefficients at different steps;
for such schedules positivity and reachable intervals must be checked for
that schedule rather than inferred from this fixed-coefficient example.

\subsection{Proof of Theorem~\ref{thm:muon_dimension_independent}}
\begin{proof}
We follow the four steps of the proof of
Theorem~\ref{thm:muon_convergence} in
Appendix~\ref{app:rank_sensitive_convergence}. Only the alignment,
update-energy, and joint-drift bounds change. Write
\[
E_t^{(l)}=N_t^{(l)}-\nabla_{(l)}F(X_t),\qquad
H_t=\sum_{l=1}^L\|\nabla_{(l)}F(X_t)\|_{\mathsf F},\qquad
b_J=\ell_J+u_J,
\]
and set $\bar H=T^{-1}\sum_{t=1}^T\mathbb E H_t$ and
$\bar E=T^{-1}\sum_{t=1}^T\sum_{l=1}^L\mathbb E\|E_t^{(l)}\|_{\mathsf F}$.

\noindent\textbf{Step 1: Descent and telescoping.}
Proposition~\ref{prop:finite_map_geometry} gives
$\langle N_t^{(l)},O_t^{(l)}\rangle\ge\ell_J\|N_t^{(l)}\|_{\mathsf F}$
and $\|O_t^{(l)}\|_{\mathsf F}\le u_J$. Splitting the true gradient into
$N_t^{(l)}-E_t^{(l)}$ and applying the reverse triangle inequality, just as
in the nuclear-norm proof, yields
\begin{equation}
\label{eq:df_descent_alignment}
\begin{aligned}
\langle\nabla_{(l)}F(X_t),O_t^{(l)}\rangle
&\ge\ell_J\|N_t^{(l)}\|_{\mathsf F}-u_J\|E_t^{(l)}\|_{\mathsf F}\\
&\ge\ell_J\|\nabla_{(l)}F(X_t)\|_{\mathsf F}
-b_J\|E_t^{(l)}\|_{\mathsf F}.
\end{aligned}
\end{equation}
These inequalities also hold when $N_t^{(l)}=0$ and $O_t^{(l)}=0$.
Moreover, $\sum_l\|O_t^{(l)}\|_{\mathsf F}^2\le L u_J^2$.
Lemma~\ref{lemma_smooth}, $z^q\le(1-q)+qz$, and
$\|\nabla F(X_t)\|_{\mathsf F}\le H_t$ therefore imply
\[
F(X_{t+1})\le F(X_t)-\gamma\ell_J H_t
+\gamma b_J\sum_l\|E_t^{(l)}\|_{\mathsf F}
+\frac{\gamma^2 L u_J^2}{2}
\bigl(\hat{\mathcal L}+q\mathcal L_1H_t\bigr).
\]
Taking expectations, summing over $t$, and using $F(X_{T+1})\ge F^*$ gives
\begin{equation}
\label{eq:df_descent_intermediate}
\left(\ell_J-\frac{\gamma L u_J^2q\mathcal L_1}{2}\right)\bar H
\le\frac D{\gamma T}+b_J\bar E
+\frac{\gamma L u_J^2\hat{\mathcal L}}{2}.
\end{equation}
This is the counterpart of \eqref{Eq.s_9}, with Frobenius stationarity
and an unweighted sum of block tracking errors.

\noindent\textbf{Step 2: The Nesterov tracking-error expansion.}
Put $\varepsilon_t^{(l)}=G_t^{(l)}-\nabla_{(l)}F(X_t)$.
The momentum recursion and $M_0^{(l)}=0$ give exactly the expansion
\eqref{eq:tracking_expansion}, which we write as
\[
E_t^{(l)}=-\beta_1^t\beta_2\nabla_{(l)}F(X_1)
+Z_t^{(l)}+V_t^{(l)},
\]
where
\[
\begin{aligned}
Z_t^{(l)}&=(1-\beta_1\beta_2)\varepsilon_t^{(l)}
+\beta_2\delta\sum_{k=1}^{t-1}\beta_1^{t-k}\varepsilon_k^{(l)},\\
V_t^{(l)}&=\beta_1\beta_2\sum_{k=2}^t\beta_1^{t-k}
\bigl(\nabla_{(l)}F(X_{k-1})-\nabla_{(l)}F(X_k)\bigr).
\end{aligned}
\]
Thus the same three terms as in \eqref{Eq.s_11} satisfy
\[
\frac1T\sum_{t=1}^T\mathbb E\|E_t^{(l)}\|_{\mathsf F}
\le {\cal T}_{1,l}+{\cal T}_{2,l}+{\cal T}_{3,l},
\]
where ${\cal T}_{1,l}$ is the averaged norm of the initialization term,
${\cal T}_{2,l}=T^{-1}\sum_t\mathbb E\|Z_t^{(l)}\|_{\mathsf F}$,
and ${\cal T}_{3,l}=T^{-1}\sum_t\mathbb E\|V_t^{(l)}\|_{\mathsf F}$.

\noindent\textbf{Step 3: Initialization, noise, and gradient drift.}
The initialization bound uses the same geometric series:
\[
\sum_l{\cal T}_{1,l}
=\frac{\beta_2}{T}\sum_{t=1}^T\beta_1^t G_{\mathsf F}
\le\frac{\beta_2G_{\mathsf F}}{T\delta}.
\]
For the noise term, conditional unbiasedness eliminates temporal cross
terms within each block. The variance bound and geometric series give
\[
\begin{aligned}
\mathbb E\|Z_t^{(l)}\|_{\mathsf F}^2
&\le\left((1-\beta_1\beta_2)^2
+\beta_2^2\delta^2\sum_{k=1}^{t-1}\beta_1^{2(t-k)}\right)\sigma_{(l)}^2\\
&\le\frac{\delta+2\beta_1(1-\beta_1\beta_2)(1-\beta_2)}{1+\beta_1}
\sigma_{(l)}^2.
\end{aligned}
\]
Jensen's inequality and $\sqrt{x+y}\le\sqrt x+\sqrt y$ imply
$\sum_l{\cal T}_{2,l}\le v_\beta S_{\mathsf F}$.
This estimate does not require independence across layers.

For drift, apply Cauchy--Schwarz across blocks before invoking joint
smoothness. In place of the rank-weighted bound in \eqref{Eq.s_14}, use
$\sum_l\|W^{(l)}\|_{\mathsf F}\le\sqrt L\|W\|_{\mathsf F}$ and
$\|X_k-X_{k-1}\|_{\mathsf F}\le\gamma u_J\sqrt L$:
\begin{equation}
\label{eq:df_tracking_drift}
\begin{aligned}
\sum_l {\cal T}_{3,l}
&\le\frac{\beta_1\beta_2\sqrt L}{T}
\sum_{t=1}^T\sum_{k=2}^t\beta_1^{t-k}
\mathbb E\|\nabla F(X_{k-1})-\nabla F(X_k)\|_{\mathsf F}\\
&\le\frac{\gamma\beta_1\beta_2L u_J}{T}
\sum_{t=1}^T\sum_{k=2}^t\beta_1^{t-k}
\mathbb E\bigl[\hat{\mathcal L}+q\mathcal L_1 H_k\bigr]\\
&\le\frac{\gamma\beta_2L u_J}{\delta}
\bigl(\hat{\mathcal L}+q\mathcal L_1\bar H\bigr).
\end{aligned}
\end{equation}
The second inequality uses Assumption~\ref{assump:layer_wise_smoothness}
with base point $X_k$ and $X_k-X_{k-1}=-\gamma O_{k-1}$.
The last exchanges the finite sums and bounds
$\beta_1\sum_{t=k}^T\beta_1^{t-k}\le\delta^{-1}$.
Combining the three terms gives the counterpart of \eqref{Eq.s_15}:
\begin{equation}
\label{eq:df_tracking_bound}
\bar E\le
\frac{\beta_2G_{\mathsf F}}{T\delta}+v_\beta S_{\mathsf F}
+\frac{\gamma\beta_2L u_J\hat{\mathcal L}}{\delta}
+\frac{\gamma\beta_2L u_Jq\mathcal L_1}{\delta}\bar H.
\end{equation}

\noindent\textbf{Step 4: Substitution and absorption.}
Substituting \eqref{eq:df_tracking_bound} into
\eqref{eq:df_descent_intermediate} and moving its last term to the left gives
\[
\begin{aligned}
&\left(\ell_J-\frac{\gamma L u_J^2q\mathcal L_1}{2}
-\frac{\gamma\beta_2L u_Jb_Jq\mathcal L_1}{\delta}\right)\bar H\\
&\qquad\le\frac D{\gamma T}
+\frac{\beta_2b_JG_{\mathsf F}}{T\delta}
+b_Jv_\beta S_{\mathsf F}
+\frac{\gamma\beta_2L u_Jb_J\hat{\mathcal L}}{\delta}
+\frac{\gamma L u_J^2\hat{\mathcal L}}{2}.
\end{aligned}
\]
Using $\beta_2\le1$ in the coefficient on the left and $\bar H\ge0$
yields \eqref{eq:df_convergence}. This completes the same descent,
tracking, and absorption argument as for Theorem~\ref{thm:muon_convergence},
with the finite-map bounds supplying all geometric constants.
\end{proof}

\subsection{Proof of Corollary~\ref{cor:muon_dimension_independent}}
\begin{proof}
As in the proof of Corollary~\ref{cor:muon_convergence}, the schedules
balance the three leading terms of the convergence bound. With
$b_J=\ell_J+u_J$, the corresponding envelope is
\[
\Phi_{\mathsf F}(\gamma,\delta)
=\frac D{\gamma T}+b_JS_{\mathsf F}\sqrt\delta
+\frac{\gamma L u_Jb_J\hat{\mathcal L}}{\delta}.
\]
For $D,S_{\mathsf F}>0$, weighted Young's inequality gives
\[
\Phi_{\mathsf F}(\gamma,\delta)
\ge\left(\frac{64D(b_JS_{\mathsf F})^2L u_Jb_J\hat{\mathcal L}}{T}\right)^{1/4},
\]
with equality when
\[
\frac{4D}{\gamma T}
=2b_JS_{\mathsf F}\sqrt\delta
=\frac{4\gamma L u_Jb_J\hat{\mathcal L}}{\delta}.
\]
These balancing equations give $\gamma\asymp T^{-3/4}$ and
$\delta\asymp T^{-1/2}$. The corollary allows any positive
$\gamma_0,\delta_0$ with these exponents; the verification below also covers
$D=0$ or $S_{\mathsf F}=0$ without using the equality conditions.

Under \eqref{eq:df_schedule}, $\delta=\delta_0T^{-1/2}$ and
\[
1-\beta_1\beta_2=\delta+\beta_1\delta^a,\qquad
v_\beta\le\sqrt\delta+\sqrt2\bigl(\delta^{(1+a)/2}+\delta^a\bigr).
\]
The condition $\Delta^{\mathsf F}_{\gamma,\beta_1}\ge\ell_J/2$ follows from
$A_JT^{-1/4}+B_JT^{-3/4}\le1$.
The sufficient threshold \eqref{eq:df_horizon} follows by the same elementary
argument as Lemma~\ref{lemma_4}; if $A_J=B_J=0$ there is nothing to check.
Substitution into \eqref{eq:df_convergence}, followed by division by
$\ell_J/2$, proves \eqref{eq:df_rate_explicit}. For $a>1/2$, all exponents
except the leading $1/4$ are strictly larger than $1/4$.
Finally $\|\nabla F(X_t)\|_{\mathsf F}\le H_t$, so the Euclidean/Frobenius
stationarity claim follows directly.

To prove \eqref{eq:df_initial_gradient}, put
$g=\|\nabla F(X_1)\|_{\mathsf F}$ and
$L_g=\mathcal L_0+\mathcal L_1g^q>0$. Applying Lemma~\ref{lemma_smooth}
to $Y=X_1-\nabla F(X_1)/L_g$ and using $F(Y)\ge F^*$ gives
$g^2\le2DL_g\le2D(\hat{\mathcal L}+q\mathcal L_1g)$.
Solving this quadratic inequality and using $G_{\mathsf F}\le\sqrt Lg$
proves the claim. This argument concerns the stated optimization assumptions,
not a network-width limit.
\end{proof}

\subsection{Scope of the angular refinement}
\label{app:angular_scope}
For one nonzero block momentum $N$ with true gradient $H$ and error $E=N-H$,
\eqref{eq:full_shape_alignment} gives the valid implication
\[
\|N\|_{\mathsf F}\ge2\bar\rho_J\|E\|_{\mathsf F}
\quad\Longrightarrow\quad
\langle H,O\rangle\ge\tfrac12\ell_J\|N\|_{\mathsf F}.
\]
If the displayed sufficient condition fails, the triangle inequality gives
\[
\|H\|_{\mathsf F}<(2\bar\rho_J+1)\|E\|_{\mathsf F}.
\]
Failure of this test is not equivalent to actual negative alignment; it only
means that this sufficient lower bound is unavailable. Nor does smallness of
one block gradient imply stationarity of the complete parameter tuple.

To make precise the stronger small-constant argument, consider the
\emph{single-block} case and separately assume spectral-direction smoothness
\begin{equation}
\label{eq:optional_spectral_smoothness}
F(X+V)\le F(X)+\langle\nabla F(X),V\rangle
+\frac{\mathcal L_{\mathrm{sp}}}{2}\|V\|_{\mathrm{op}}^2.
\end{equation}
This is an additional hypothesis, not a consequence of
Assumption~\ref{assump:layer_wise_smoothness} with a dimension-independent
constant. Define $h_\gamma=\mathcal L_{\mathrm{sp}}\gamma c_J^2/\ell_J$.
If $\|N_t\|_{\mathsf F}\ge2\bar\rho_J\|E_t\|_{\mathsf F}+h_\gamma$,
\eqref{eq:full_shape_alignment} and
\eqref{eq:optional_spectral_smoothness} yield
\[
F(X_{t+1})\le F(X_t)-\tfrac12\gamma\ell_J\|N_t\|_{\mathsf F}.
\]
On any sample path for which $\max_{t\le T}\|E_t\|_{\mathsf F}\le\bar e$,
either this condition first fails, or it holds at every step. In the first
case the gradient at the failed test is at most
$(2\bar\rho_J+1)\bar e+h_\gamma$; in the second case telescoping gives
$\min_t\|\nabla F(X_t)\|_{\mathsf F}\le2D/(\gamma\ell_JT)+\bar e$.
Therefore the precise deterministic certificate is
\begin{equation}
\label{eq:conditional_angular_certificate}
\min_{t\le T}\|\nabla F(X_t)\|_{\mathsf F}
\le\frac{2D}{\gamma\ell_JT}
+(2\bar\rho_J+1)\bar e
+\frac{\mathcal L_{\mathrm{sp}}\gamma c_J^2}{\ell_J}.
\end{equation}
An averaged or pointwise expected error bound is not a bound on
$\max_t\|E_t\|_{\mathsf F}$, and it cannot simply be substituted for
$\bar e$. Initialization transients must also be included in any such uniform
bound. Equation~\eqref{eq:conditional_angular_certificate} is thus a
conditional geometric certificate, not a second stochastic convergence
claim under the three stated assumptions.

The size of $u_J$ cannot in general be removed from a uniform Frobenius
update bound by sharpening a scalar inequality. Let
$A_0=\phi_J'(0)>0$ and take rank $r$ with all normalized singular values
$s_i=r^{-1/2}$. Then
\[
A(N)=\sqrt r\,\phi_J(r^{-1/2})\longrightarrow A_0,\qquad
B(N)=r\phi_J(r^{-1/2})^2\longrightarrow A_0^2.
\]
Thus any uniform bound $\|O\|_{\mathsf F}\le U$ needs $U\ge A_0$, and
any $B\le K A$ needs $K\ge A_0$. For the fixed quintic, $A_0=u_J$.
The large universal magnitude constant and the absence of dimension factors
are distinct issues.

\end{document}